\documentclass[twoside]{article}

\usepackage{hyperref}
\usepackage{authblk}

\usepackage{url}
\usepackage{natbib}
\usepackage{amssymb,amsfonts,amsmath,amsthm}
\usepackage{bm}
\usepackage{graphicx} 
\usepackage{caption}
\usepackage{subcaption}
\usepackage{algorithm}
\usepackage{algpseudocode}
\newtheorem{theorem}{Theorem}
\newtheorem{lemma}{Lemma}

\DeclareMathOperator*{\argmax}{\arg\!\max}
\DeclareMathOperator*{\argmin}{\arg\!\min}

\begin{document}

%

%


\title{Estimating the Accuracies of Multiple Classifiers Without Labeled Data}

\author[1,*]{Ariel Jaffe}
\author[1,**]{Boaz Nadler}
\author[2,3,***]{Yuval Kluger}
\affil[1]{\footnotesize Dept. of Computer Science and Applied Mathematics, Weizmann Institute of Science, Rehovot Israel 76100}
\affil[*]{\url{ariel.jaffe@weizmann.ac.il}} 
\affil[**]{\url{boaz.nadler@weizmann.ac.il}}
\affil[2]{Dept. of Pathology,  Yale University, School of Medicine, New Haven, CT 06520}
\affil[3]{NYU Center for Health Informatics and Bioinformatics New York University, Langone Medical Center, 227 East 3030\textsuperscript{th} Street, New York, NY 10016, USA}
\affil[***]{\url{yuval.kluger@yale.edu}}
\date{}

%


\maketitle

\begin{abstract}
  In various situations one is given only the predictions of multiple classifiers over a large unlabeled test data. This scenario raises the following questions: Without any labeled data and without any a-priori knowledge about the reliability of these different classifiers, is it possible to consistently and computationally efficiently estimate their accuracies? Furthermore, also in a completely unsupervised manner, can one construct a more accurate unsupervised ensemble classifier? In this paper, focusing on the binary case, we present simple, computationally efficient algorithms to solve these questions.   Furthermore, under standard classifier independence assumptions, we prove our methods are consistent and study their asymptotic error. Our approach is spectral, based on the fact that the off-diagonal entries of the classifiers' covariance matrix and 3-d tensor are rank-one. We illustrate the competitive performance
of our algorithms via extensive experiments on both artificial and real datasets. 

 \end{abstract}

\section{Introduction}

Consider a classification problem from an instance space \(\mathcal X\) to an output label set \(\mathcal Y=\{1,\ldots,K\}\). In contrast to the classical supervised setting, in various contemporary applications, one has access only to the predictions  of multiple experts or classifiers over a large number of unlabeled instances. 
Moreover, the reliability of these experts may be unknown, and at test time there is no labeled data to assess it. This occurs for example  when due to privacy considerations each classifier is trained with its own possibly proprietary labeled data, unavailable to us.  
Another scenario is crowdsourcing, where an annotation task over many instances is distributed to many annotators whose reliability is a-priori unknown, see for example \cite{Welinder_2010,Whitehill_2009,Sheshadri_SQUARE_2013}. This setup, denoted as unsupervised-supervised learning in \cite{Donmaz_2010}, appears in several  other application domains,  including decision science, economics and  medicine, see \cite{Snow,Raykar,Parisi_2014}.  

Given only the $m\times n$\ matrix \(Z,\) or a significant part of it, with \(Z_{ij}=f_{i}(x_j)\) holding the predictions of the given \(m\) classifiers over  \(n\) instances, and without any labeled data,  two fundamental questions arise: (i) Under the assumption that different classifiers make independent errors, is it possible to consistently estimate the accuracies of the \(m\) classifiers in a computationally efficient way; and (ii) is it possible to construct, again by some computationally efficient procedure, an unsupervised ensemble learner, more accurate than most if not all of the original \(m\) classifiers.

The first question is important in cases where obtaining the predictions of these \(m\) classifiers is by itself an expensive task, and after collecting a certain number of instances and their predictions, we wish to pick only a few of the most accurate ones, see \cite{Rokach_Ensemble_Pruning_2009}. The second question, also known as \textit{offline consensus}, is of utmost importance in improving the quality of automatic decision making systems based on multiple sources of information.

Beyond the simplest approach of majority voting, perhaps the first to define and address these questions were \cite{Dawid_79}. With the increasing popularity of crowdsourcing and large scale expert opinion systems, the last years have seen a surge of interest in these problems, see   \cite{Sheng_2008,Whitehill_2009},
  \cite{Raykar}, \cite{PLATANIOS_14} and references therein. Yet, the most common methods to address questions (i) and (ii) above are based on the expectation maximization (EM) algorithm, already proposed in this context by Dawid and Skene, and whose only guarantee is convergence to a local maxima. 

Two recent exceptions, proposing spectral (and thus computationally efficient) methods with strong consistency guarantees are   \cite{Karger_2011} and \cite{Parisi_2014}.  \cite{Karger_2011} assume a spammer-hammer model, where each classifier is either perfectly correct or totally random and develop a spectral method to detect which one is which. \cite{Parisi_2014}\ derive a spectral approach to address questions (i) and (ii)\ above in the context of binary classification. Their approach, however, has several limitations. First, they do not actually estimate each classifier sensitivity and specificity, but only show how to consistently rank them according to their balanced accuracies. Second, their unsupervised learner assumes that all classifiers have balanced accuracies close to 1/2 (random). Hence, their ensemble learner  may be suboptimal, for example, when few classifiers are significantly more accurate than all others. 

In this paper we extend and generalize the results of \cite{Parisi_2014} in several directions and make the following contributions: In Sec. \ref{sec:psi_eta}, focusing on the binary case, we present a simple spectral method to estimate the sensitivity and specificity of each classifier, assuming the class imbalance is  known.  Hence, the problem boils down to estimating a single-scalar -- the class imbalance. In Section \ref{sec:estimate_b} we present two different methods to do so. First, in Sec. \ref{sec:b_tensor}, we prove
that the off-diagonal elements of the $m \times m $ covariance matrix and the $ m \times m \times m$ joint covariance tensor of the set of classifiers are both rank 1. Moreover the covariance matrix and tensor share the same eigenvector but with different eigenvalues, from which the class imbalance can be extracted by a simple least-squares procedure. In Sec. \ref{sec:algo_likelihood}, we devise a second algorithm to estimate the class imbalance by a restricted likelihood approach. The maxima of this function is attained at the class imbalance, and can thus be found by a one-dimensional scan.
Both algorithms are computationally efficient, and under the assumption that classifiers make independent errors, are also proven to be consistent. For the first method, we also prove it is rate optimal with asymptotic error  \(\mathcal O_{P}(1/\sqrt{n})\), where $n$ is the number of unlabeled samples. Our work thus provides a simple and elegant solution to the long-standing problem originally posed
by Dawid and Skene [2], whose previous solutions were mostly
based on expectation maximization approaches
to the full likelihood function.  


In Sec. \ref{sec:multiclass} we consider the multiclass case. Building upon standard reductions from multiclass to binary, we devise a method to estimate the class probabilities and the diagonal entries of the confusion matrices of all classifiers. We also prove that in the multiclass case, using only the first and second moments of these binary reductions, it is in general not possible to estimate all entries of the confusion matrices of all classifiers.
This motivates the development of tensor or higher order methods to solve the multi-class case, as for example in \cite{Zhang_2014}. In Sec. \ref{sec:simlations} we illustrate our methods  on both real and artificial data. The results on  real data show that our proposed ensemble learner achieves a  competitive  performance even in practical scenarios where the  assumption of independent classifiers' errors does not hold precisely.

\paragraph{Related Work} Under the assumption that all classifiers make independent errors, the crowdsourcing problem we address is equivalent to learning a mixture of discrete product distributions. This problem was studied, among others, by \cite{Freund_1999} for the case of $k=2$ distributions,  and by \cite{Feldman_2008} for $k >2$. Important observations regarding the low-rank spectral structure of the second and third moments of such distributions were made by  \cite{Anand_2012_a,Anand_2012_b}.
Building upon these results, recently \cite{Jain_2014} and \cite{Zhang_2014}, devised computationally efficient algorithms to estimate the parameters of the mixture of product distributions, which are equivalent to the confusion matrices and class probabilities in our problem. 

Our first method to estimate the class imbalance in the binary case using the mean-centered 3-d tensor is closely related to these works, with some notable differences. One key difference is that the above works study non-centered tensors of classifiers' outputs, and hence for a k-class problem, need to resolve the structure of rank-k tensors. In contrast, we work with centered matrices and tensors. In the binary case with \(k=2,\) we thus obtain a simpler rank-1 tensor, which we do not even need to decompose, but only extract a single scalar from it. A second difference is that the above methods require stronger assumptions on the classifiers. For example, \cite{Zhang_2014} divide the classifiers into groups and assume that within each group, on average classifiers are better than random. Due to these differences, our resulting algorithm is significantly simpler.  

Our second algorithm for estimating the class imbalance, based on a restricted likelihood approach is totally different from these tensor-based works, as it requires only a spectral decomposition of the classifiers' covariance matrix, and then optimizes a 1-d function of the full likelihood of the data. 
On both simulated and real data, this second approach had at least as good as, and in some cases better accuracy compared to the tensor based method. Finally, while we focus on classification, our algorithms may also be of interest to  learning a mixture of  discrete product distributions.

\section{Problem Setup}
        \label{sec:setup}
We consider the following binary classification problem, as also studied in several works  (\cite{Dawid_79,Raykar,Parisi_2014}).
Let  $\mathcal{X}$ be an instance space with an output space $\mathcal{Y}=\{-1,1\}$. A labeled instance $(x,y) \in  \mathcal{X} \times \mathcal{Y}$ is a realization of the random variable $(X,Y),$ which has an unknown probability density $p(x,y), $ and $X$ and $Y$ marginals $p_X(x)$ and $p_Y(y),$ respectively.
We further denote by $b$ the class imbalance of $Y$,
\[
b = \Pr(Y=1)-\Pr(Y=-1)=p_{Y}(1)-p_Y(-1).
\]

Let $\{f_i\}_{i=1}^m$ be  $m\geq 3$  classifiers operating on $\mathcal{X}$. 
In this binary setting,
the accuracy of the $i$-th classifier is fully specified by its sensitivity $\psi_i$ and specificity $\eta_i$,
\begin{eqnarray}
\psi_i=& \Pr\left(f_i(X)=1|Y=1\right) \notag \\ 
 \eta_i=& \Pr\left(f_i(X)=-1|Y=-1\right) \notag.
\end{eqnarray}
For future use, we denote by $\pi_i$ its balanced accuracy, 
\[
\pi_i = (\psi_i+\eta_i)/2.
\]
 In this paper we consider the following totally \textit{unsupervised} scenario.   Let $Z$ be a $m \times n $ matrix with entries $Z_{ij}=f_i(x_j),i=1,\ldots, m,j=1,\ldots,n$, where $f_i(x_j)$ is the label predicted at instance $x_j$ by classifier $f_i$. In particular, we assume no prior knowledge about the \(m\) classifiers, so their accuracies (sensitivities $\psi_{i}$ and specificities $\eta_{i}$) are all unknown.
 
 Given only the matrix $Z$ of binary predictions\footnote{For simplicity of exposition, we assume the matrix is fully observed. While beyond the scope of this paper, our proposed methods and theory continue to hold if few entries are missing (at random), such that accurate estimates of various means,  covariances and tensors, as detailed in Sections 3-4 are still possible. }, we consider the following two problems: (i) consistently and computationally efficiently estimate the sensitivity and specificity of each classifier, and (ii) construct a more accurate ensemble classifier. As discussed below, under certain assumptions, a solution to the first problem readily yields a solution to the second one.

To tackle these problems, we make the following three assumptions: (i) The $n$ instances $x_j$ are i.i.d. realizations from the marginal $p_X(x)$. (ii) The $m$ classifiers are conditionally independent. That is, for every pair of classifiers $f_i,f_j $ with $i \neq j$ and for all labels $a_i,a_j \in \{-1,1\}$,
\begin{multline}
\Pr(f_i=a_i,f_j = a_j|Y=y) = \\ \Pr(f_i=a_i|Y=y)\Pr(f_j=a_j|Y=y).
\label{eq:pair}
\end{multline} 
(iii) Most of the classifiers are better than random, in the sense that for more than half of all classifiers, $\pi_i>0.5$.
Note that (i)-(ii) are standard assumptions in both the supervised and unsupervised settings, see \cite{Dietterich,Dawid_79,Raykar,Parisi_2014}.
Assumption (iii) or a variant thereof is needed, given an inherent \(\pm 1\) sign ambiguity in this fully unsupervised problem.
\section{Estimating ${\psi}$  and $\eta$ with a known class imbalance.}
        \label{sec:psi_eta}

For some classification problems, the class imbalance $b$ is known. One example is in epidemiology, where the overall prevalence of a certain disease in the population is known, and the classification problem is to predict its presence, or future onset, in individuals given their observed features (such as blood results, height, weight, age, genetic profile, etc).

Assuming $b$ is known, \cite{Donmaz_2010} presented a simple method to estimate the error rates of all classifiers under a symmetric noise model, where $\psi_i=\eta_i$ for all $i$, and EM methods in the general case, see also \cite{Raykar}. We instead
build upon the spectral approach in \cite{Parisi_2014}, and present a computationally efficient method to consistently estimate the sensitivities and specificities of all \(m\) classifiers.
To motivate our approach, it is instructive to study the limit of an infinite unlabeled set size, \(n\to\infty\), where the mean values of the
 classifiers \(\mu_i=\mathbb{E}[f_i(X)]\), and their \(m\times m\) population covariance matrix 
\(R=\mathbb{E}\left[(f_i(x)-\mu_i)(f_j(x)-\mu_j) \right]\), 
are all perfectly known. 

The following two lemmas show that $R$ and $\{\mu_i\}_{i=1}^m$ contain the information needed to extract the specificities and sensitivities of the \(m\) classifiers. Lemma \ref{lem:v} appeared in \cite{Parisi_2014}, and implies that given the value of $b$ one may compute the balanced accuracies of all classifiers. Lemma \ref{lem:psi_eta} is new and shows how to extract their sensitivities and specificities. Its proof appears in the appendix.
\begin{lemma}
The off diagonal elements of the matrix $R$ are identical to those of a rank one matrix $\mathbf v \mathbf v^T$, whose vector $\mathbf v$, up to a $\pm1 $ sign ambiguity, is equal to
\begin{equation}
\mathbf v = \sqrt{1-b^2}(2 \boldsymbol \pi-1),
\label{eq:v}
\end{equation}
where the vector $\boldsymbol \pi = (\pi_1, \ldots,\pi_m)$ contains the balanced accuracies of the $m$ classifiers.
\label{lem:v}
\end{lemma}
\begin{lemma}
Given the class imbalance \(b,\) the vector $\boldsymbol \mu=(\mu_1, \ldots,\mu_m)$ containing the mean values of the $m$ classifiers, and   $\mathbf  v$ of Eq.  \eqref{eq:v}, the values of $\boldsymbol \psi=(\psi_1,\ldots,\psi_m)$ and  $\boldsymbol \eta=(\eta_1,\ldots,\eta_m)$ with the specificities and sensitivities of the \(m\) classifiers are given by
\begin{equation}
\boldsymbol \psi =  \tfrac{1}{2}\Big(1+\boldsymbol \mu +\mathbf v \sqrt{\tfrac{1-b}{1+b}}\Big),
\boldsymbol \eta =  \tfrac{1}{2}\Big(1- \boldsymbol\mu +\mathbf v \sqrt{\tfrac{1+b}{1-b}}\Big). \label{eq:psi_eta}
\end{equation}
\label{lem:psi_eta}
\end{lemma}
To uniquely recover $\mathbf v$ from the off-diagonal entries of $R$, we further assume that at least three classifiers have different balanced accuracies, which are all different from 1/2 (so $2\pi_{i}-1\neq 0)$. 
In practice, 
the quantities $\{\mu_i\}_{i=1}^m$, $R$ and consequently the eigenvector ${\bf v}$ are all unknown. We thus estimate them from the given data, and plug into Eq. \eqref{eq:psi_eta}. Let us denote by $\hat{\boldsymbol{\mu}}$  and $\hat R$ the sample mean and covariance matrix of all classifiers, whose entries are given by
\begin{eqnarray}
\hat \mu_i &=& \frac{1}{n}\sum_{k=1}^n f_i(x_k), \\ \notag \hat r_{ij} &=& \frac{1}{n-1}\sum_{k=1}^n (f_i(x_k)-\hat\mu_i)(f_{j}(x_k)-\hat\mu_{j}).
\label{eq:rii}
\end{eqnarray} 
Estimating the vector \(\mathbf v\) from the noisy matrix \(\hat R\) can be cast as a low-rank matrix completion problem. 
  \cite{Parisi_2014} present several methods to construct such an estimate $\hat{\mathbf v}$, and resolve its inherent  $\pm1$ sign ambiguity, via assumption (iii).
Inserting $\hat{\boldsymbol{\mu}}$ and $\hat{\mathbf{v}}$ into \eqref{eq:psi_eta}, gives the following estimates for $\boldsymbol{\psi}$ and $\boldsymbol{\eta}$,
\begin{equation}
\hat{\boldsymbol \psi} =  \tfrac{1}{2}\Big(1+\hat{\boldsymbol \mu} + \hat{\mathbf v} \sqrt{\tfrac{1-b}{1+b}}\Big),
\hat{\boldsymbol \eta} =  \tfrac{1}{2}\Big(1-\hat{\boldsymbol \mu} + \hat{\mathbf  v} \sqrt{\tfrac{1+b}{1-b}}\Big).
\label{eq:psi_eta_est}
\end{equation}

The following lemma, proven in the appendix, presents some statistical properties of $\hat{\boldsymbol{\psi}}$ and $\hat{\boldsymbol{\eta}}$.
\begin{lemma} Under assumptions (i)-(iii) of Section \ref{sec:setup}, 
$\hat{\boldsymbol{\psi}}$ and $\hat{\boldsymbol{\eta}}$ are consistent estimators of $\boldsymbol{\psi}$ and $\boldsymbol{\eta}$. Furthermore, 
as $n \to \infty$,  
 \begin{equation}
\hat {\psi}_i = \psi_i + \mathcal{O}_P\left(\frac{1}{\sqrt{n}}\right), \quad \hat {\eta}_i = \eta_i + \mathcal{O}_P\left(\frac{1}{\sqrt{n}}\right).
\end{equation}
\label{lem:psi_eta_est}
\end{lemma}
In summary, assuming the class imbalance \(b\) is known, 
Eq. \eqref{eq:psi_eta_est} gives a computationally efficient way to estimate the sensitivities and specificities of all classifiers. Lemma \ref{lem:psi_eta_est} ensures that this approach is also consistent.  In the next section we show that the assumption of explicit knowledge of \(b\) can be removed, whereas in Section \ref{sec:multiclass}
we show that a similar approach can also (partly) handle the multiclass case.  

\subsection{Unsupervised Ensemble Learning}
\label{sec:new_sml}
We now consider the second problem discussed in Section \ref{sec:setup}, the construction of an unsupervised ensemble learner. To this end, note that under the stronger assumption that all classifiers make independent errors, the likelihood of a label \(y\) at an instance \(x\) with predicted labels \(f_1(x),\ldots,f_m(x)\) is
\begin{equation}
\mathcal L(f_1(x),\ldots,f_m(x))\,|\,y)=\prod_{i=1}^m \Pr(f_i(x)\,|\,y).
        \label{eq:full_likelihood}
\end{equation}
In Eq. (\ref{eq:full_likelihood}), the \textit{i}-th term \(\Pr(f_i(x)|y)\) depends on the specificity and sensitivity \(\psi_i\) and $\eta_i$ of the $i$-th classifier. 
While the likelihood is non-convex in  $\psi_i,\eta_i$ and $y$, if the former are known, there is a closed form solution for the maximum-likelihood value of the class label, 
\begin{equation}
        \label{eq:y_ML}
\hat y^{(\mbox{\tiny ML})} = \mbox{sign}\left({\textstyle\sum_{i}}\, f_i(x)\ln \alpha_i +\ \ln \beta_i\right)
\end{equation}
where
\begin{equation}
        \label{eq:alpha_beta}
\alpha_i = \frac{\psi_i\eta_i}{(1-\psi_i)(1-\eta_i)},\quad \beta_i = \frac{\psi_i(1-\psi_i)}{\eta_i(1-\eta_i)}.
\end{equation}
\cite{Parisi_2014}, assumed all classifiers are close to random, and via a Taylor expansion near $\psi=\eta=1/2$, showed that $\beta$ is approximately zero, and $\alpha_i\approx 1+4(2\pi_i-1)$. Plugging these into Eq. (\ref{eq:y_ML}), they  derived the following spectral meta-learner (SML), 
\begin{equation}
\hat y^{(\mbox{\tiny SML})} = \mbox{sign}\left(\textstyle\sum_i f_i(x)\hat v_i\right).
\end{equation}
Their motivation was that they only had estimates of the vector \(\mathbf v\), which according to Eq. (\ref{eq:v})  is proportional to \((2\boldsymbol\pi-1)\). Since we consistently estimate the individual specificities and sensitivities of the \(m\) classifiers, we  suggest to plug in these estimates directly into Eqs. (\ref{eq:alpha_beta}) and (\ref{eq:y_ML}). Our improved spectral approach, denoted i-SML, yields a more accurate ensemble learner when few classifiers are significantly better than random, so the linearization around \(\psi=\eta=1/2\) is inaccurate. We present such examples in Sec. \ref{sec:simlations}.
Finally, we note that as in \cite{Parisi_2014} and \cite{Zhang_2014}, we may use our i-SML as a starting guess for EM methods that maximize the full likelihood. 

\section{Estimation of the class imbalance}
        \label{sec:estimate_b}
We now consider the problem of estimating $\boldsymbol \psi$ and $\boldsymbol \eta$ when the class imbalance $b$ is unknown.
Our proposed approach is to first estimate $b$, and then plug this estimate into Eq. \eqref{eq:psi_eta_est}. 
We present two different methods to estimate the class imbalance. The first  uses the covariance matrix and the 3-dimensional covariance tensor of all $m$ classifiers.
The second method exploits properties of the likelihood function.
As detailed below,  both methods are computationally efficient, but require stronger assumptions than Eq.\eqref{eq:pair} on  independence of classifier errors to prove their consistency. 
\subsection{Estimation via the 3-D covariance tensor}
\label{sec:b_tensor}

For the method derived in this subsection, we assume that the classifiers are conditionally independent in triplets. That is, for every  $f_i,f_j , f_k$ with $i \neq j \neq k$ and for all labels $a_i,a_j,a_k \in \{-1,1\}$,
\begin{multline}
\Pr(f_i=a_i,f_j = a_j,f_k = a_k|y) = \\ \Pr(f_i=a_i|y)\Pr(f_j=a_j|y)\Pr(f_k=a_k|y).
\label{eq:triplet}
\end{multline}
Let $T=(T_{ijk})$ denote the 3-dimensional covariance tensor of the $m$ classifiers $\{f_i(X)\}_{i=1}^m$,
\begin{equation}                                        
T_{ijk} = \mathbb E\left[(f_i(X)-\mu_i)(f_j(X)-\mu_j)(f_k(X)-\mu_k)\right].
\label{eq:tensor}
\end{equation}
The following lemma, proven in the appendix, provides the relation between the tensor $T$, the class imbalance $b$ and the balanced accuracies of the $m$ classifiers.
\begin{lemma} 
\label{lem:T}
Under assumption \eqref{eq:triplet}, the following holds for all $i \neq j \neq k$, 
\begin{equation}
T_{ijk} = -2b(1-b^2)(2\pi_i-1)(2\pi_j-1)(2\pi_k-1).
\label{eq:r_ijk}
\end{equation}
\end{lemma}

According to \eqref{eq:r_ijk}, the off diagonal elements of $T$ (with $i \neq j \neq k$) correspond to a \textit{rank one tensor},
\begin{equation}
T = \mathbf w \otimes \mathbf w \otimes \mathbf w,
\end{equation}
where $\otimes$ denotes the outer product and the vector $\mathbf w \in \mathbb{R}^m$ is equal to
\begin{equation}
\mathbf w = \left(-2b(1-b^2)\right)^{\frac{1}{3}}\cdot (2\boldsymbol{\pi}-1).
\label{eq:w}
\end{equation}
Note that unlike the vector $\mathbf v$ of the covariance matrix $R$, 
there is no sign ambiguity in the vector $\mathbf w$.

Moreover, comparing Eqs. \eqref{eq:v} and \eqref{eq:w}, the vectors $\mathbf v$ of $R$ and $\mathbf w$ of $T$ are both proportional to  $(2\boldsymbol \pi -1)$, where the proportionality factor depends on the class imbalance $b$. Hence, $\mathbf w= \alpha(b)^{1/3} \,\mathbf v$, and 
\begin{equation}
T = \alpha(b)\, \mathbf v \otimes \mathbf v \otimes \mathbf v
        \label{eq:T_alpha}
\end{equation}
where 
$
\alpha(b) = (-2b)/\sqrt{1-b^2}.
$
Inverting this expression yields the following relation,
\begin{equation}
b = - \alpha/\sqrt{4+\alpha^2}.
\label{eq:b}
\end{equation} 
Eq. \eqref{eq:b} thus shows, that in our setup, as $n \to \infty$, the first three moments of the data ($\boldsymbol \mu,R,T$) are sufficient to determine both the class imbalance and the sensitivities and specificities of all $m$ classifiers.

In practice, the tensor $T$ is unknown, though it can be estimated from the observed data by
\begin{equation}
\hat T_{ijk} = \frac{1}{n}\sum_{l=1}^n(f_i(x_l)-\hat\mu_i)(f_j(x_l)-\hat\mu_j)(f_k(x_l)-\hat\mu_k).
\label{eq:T}
\end{equation}
Given an estimate \(\hat{\mathbf v}\) from the matrix \(\hat R\), the scalar \(\alpha\) of Eq. (\ref{eq:T_alpha}) is estimated 
by least squares, 
\begin{equation}
\hat{\alpha} = \argmin_\alpha \sum_{i<j<k} \left( \hat T_{ijk} - \alpha\, \hat v_i\hat v_j \hat v_k\right)^2    \label{eq:hat_alpha}
.\end{equation}
A summary of the steps to estimate the class imbalance with the 3 dimensional tensor appears in Algorithm 1. The following lemma shows that this method yields an asymptotic error of 
$\mathcal O_P(1/\sqrt{n})$. This error rate is optimal
since even if we knew the ground truth labels $y_i$, estimating $b$ from them would still incur such an error rate.

\begin{lemma}\label{lemma:b_hat_tensor}
Let $\hat\alpha$ be given by Eq. (\ref{eq:hat_alpha}) and let $\hat b_n$ be the plug-in estimator from Eq. (\ref{eq:b}). Then, 
\begin{equation}
\hat b_n = b + \mathcal O_P\left(1/\sqrt{n}\right). 
\end{equation}
Consequently the plug-in estimators $\hat \psi_i,\hat \eta_i$ in Eq. (\ref{eq:psi_eta_est})\ also have
the same asymptotic error $ \mathcal O_P(1/\sqrt{n})$. 
\end{lemma}
The proof of Lemma \ref{lemma:b_hat_tensor} appears in the appendix. Following it are some remarks regarding the accuracy of various estimates as a function of the number of classifiers and their accuracies. A\ detailed study of this issue is beyond the scope of this paper. 


\begin{algorithm}[t]
\caption{Estimating class imbalance with the 3dimensional covariance tensor}
\begin{algorithmic}[1]
\State Estimate covariance matrix $R$ by Eq. \eqref{eq:rii}. 
\State Estimate $\mathbf v$ from the off diagonal entries of  $\hat{R}$ (see appendix).
\State Estimate the 3 dimensional tensor $T$ by Eq. \eqref{eq:T}.
\State Estimate $\alpha$ via Eq. (\ref{eq:hat_alpha}) and  $b$ via Eq. \eqref{eq:b}.
\end{algorithmic}
\end{algorithm}

\label{sec:algoT}

\subsection{A restricted-likelihood approach}

\label{sec:algo_likelihood}
The algorithm in Section \ref{sec:b_tensor} relied only on the first three moments of the data. We now present a second method to estimate the class imbalance, based on a restricted likelihood function of all the data. 
This method  is potentially more accurate, however it requires the following stronger assumption of joint conditional independence of all $m$ classifiers, 
\begin{equation}
\Pr(f_1\!=\!a_1,\ldots,f_m\! =\! a_m|y) = \prod_{i=1}^m \Pr(f_i\!=\!a_i|y).
        \label{eq:all_f_independent}
\end{equation}

It is important to note that under this assumption, the  problem at hand is equivalent to learning a mixture of two product distributions, addressed in \cite{Freund_1999}. For this problem, several recent works suggested spectral tensor decomposition approaches, see  \cite{Anand_2012_a,Jain_2014,Zhang_2014}.

In contrast, we now present a totally different approach, not based on tensor decompositions. Our starting point is Eq. \eqref{eq:psi_eta_est} which provides consistent estimates of $\boldsymbol \psi$ and $\boldsymbol \eta$ given the class imbalance $b$. In particular, 
any guess $\tilde b$ of the class imbalance, yields corresponding guesses for the sensitivities and specificities of all \(m\) classifiers, 
$\hat{\boldsymbol\psi}(\tilde b)$ and $\hat{\boldsymbol\eta}(\tilde b)$. As described below, our approach is to construct a suitable functional \(\hat G_n (Z|\tilde b)\), that depends on both $\tilde b$ and on the observed data \(Z\), whose maxima as a function of \(\tilde b\), as $n\to\infty$ is attained at the true class imbalance \(b\).  

To this end, let \(\mathbf f(x)=\left( f_1(x),\ldots, f_m(x)\right)\) denote the vector of labels predicted by the \(m\) classifiers at an instance \(x\). We define the following approximate log-likelihood, assuming class imbalance \(\tilde b \)
\begin{equation}
\hat g_n(\mathbf f(x)|\tilde b ) = \log \Pr\left(\mathbf f(x)|\hat{\boldsymbol \psi}(\tilde b),\hat{\boldsymbol \eta}(\tilde b),\tilde b\right) 
\label{eq:g_def}
\end{equation} 
where $\hat{\boldsymbol\psi}$ and $\hat{\boldsymbol\eta}$ are given by Eq. (\ref{eq:psi_eta_est}), and an   expression for the above probability is given in Eq. (\ref{eq:Prob_f_b}) in the appendix.   
Our functional  \(\hat G_{n} (Z|\tilde b)\)
is the average of  $\hat g_n(\mathbf f(x)|\tilde b)$ over all instances $x_j$,
\begin{equation}
\hat G_{n} (Z|\tilde b)= \frac{1}{n}\sum_{j=1}^n \hat g_n(\mathbf f(x_j)|\tilde b).
 \label{eq:G_def}
\end{equation}
Note that the estimates of $\boldsymbol\psi,\boldsymbol\eta$ in Eq. \eqref{eq:psi_eta_est} become numerically unstable for $b$ close to \(\pm 1\). Hence, in what follows we assume there is an a-priori known \(\delta>0\), such that the true class imbalance $b\in[-1+\delta,1-\delta]$. The estimate of the class imbalance is then defined as 
\begin{equation}
\hat b_{n} = \argmax _{\tilde b\in[-1+\delta,1-\delta]}
\hat G_{n} (Z|\tilde b).
        \label{eq:hat_b_G}
\end{equation}
 
To justify Eq. (\ref{eq:hat_b_G}), it is again constructive to consider the limit \(n\to\infty\). First, for any \(\tilde b\in[-1+\delta,1-\delta]\),
the convergence of $\hat{\boldsymbol \psi}(\tilde b)$ and $\hat{\boldsymbol \eta}(\tilde b)$ to ${\boldsymbol \psi}(\tilde b)$ and $\boldsymbol \eta(\tilde b)$, respectively, implies that at any instance \(x\),
\[
\lim_{n\to\infty}\hat g_n(\mathbf{f}(x)|\tilde b)=g(\mathbf{f}(x)|\tilde b) \equiv \log\!\,  \Pr(\mathbf f(x)|{\boldsymbol \psi}(\tilde b),{\boldsymbol \eta}(\tilde b),\tilde b).
\]
Next, since the $n$ instances $x_j$ are i.i.d, by the law of large numbers, combined with the delta method
\begin{equation}
\lim_{n \to \infty} 
\hat G_n(Z|\tilde b) = G(\tilde b) \equiv \mathbb{E}_{(X,Y)}\left[g(\mathbf f(X)|\tilde b)\right].
\label{eq:E_g}
\end{equation}
The following theorem, proven in the appendix, shows that the maxima of $G(\tilde b)$ is obtained at the true class imbalance $\tilde b=b$, and that \(\hat b_n\to b\) in probability. 
\begin{theorem}
Assume all classifier errors are independent, so Eq. (\ref{eq:all_f_independent}) holds. Let \(\epsilon,\delta>0 \) be a-priori known, such that classifiers sensitivities and specificities satisfy $\epsilon < \psi_i,\eta_i<1\!-\!\epsilon$,
and  \(b\in[-1+\delta,1-\delta].\) Then, 
\begin{equation}
        \label{eq:b_max_Eg}
b = \argmax_{\tilde b\in[-1+\delta,1-\delta]} \mathbb{E}_{(X,Y)}\left[g(\mathbf f(X)|\tilde b)\right]
\end{equation}
and as $n\to\infty$ the estimate $\hat b_n$ of Eq. (\ref{eq:hat_b_G}) converges to $b$  in probability. 
\label{Thm:g}
\end{theorem}
Note that since \(\hat b_n\) is the maximizer of a restricted likelihood, its convergence to \(b\) is not a direct consequence of the consistency of ML estimators. Instead, what is needed is uniform convergence in probability of \(\hat G_n(\tilde b)\) to $G(\tilde b)$, see \cite{Newey_1991} and appendix.
Also note that even though \(\hat G_n(\tilde b)\) is not necessarily concave, finding its global maxima requires optimization of a smooth function of only one variable. 

Algorithm 2 summarizes the method to estimate \(b\) by the restricted-likelihood method.
\begin{algorithm}[t]
  \caption{Estimating the class imbalance using the restricted likelihood functional}\label{euclid}
  \begin{algorithmic}[1]
      \State{Estimate the mean values $\{\hat \mu_i\}_{i=1}^m$, 
      the covariance matrix $\hat R$, and the vector $\mathbf{\hat{v}}$.}  
      \For{$\tilde b \in (-1+\delta, 1-\delta)$}
      \State\parbox[t]{\dimexpr\linewidth-\algorithmicindent}{Estimate  $\hat{\boldsymbol{\psi}}(\tilde b)$ and $\hat{\boldsymbol{\eta}}(\tilde b)$ via Eq. \eqref{eq:psi_eta_est}.}
      \State{Calculate $\hat G_n(Z|\tilde b)$ by Eqs. \eqref{eq:g_def} and \eqref{eq:G_def}}.
      \EndFor
      \State{Estimate $b$ by Eq. (\ref{eq:hat_b_G}).}
  \end{algorithmic}
\end{algorithm}
This algorithm scans possible values of \(\tilde b\), where each evaluation of $\hat G_n$ requires \(O(mn)\) operations. Since \(\hat g_{n}\) and consequently $\hat G_{n}$ are smooth functions of \(\tilde b\)
in  $(-1 + \delta,1-\delta)$,  
the finite grid of values of \(\tilde b\) can be of size polynomial in \(n\) and the method is computationally efficient. 
\label{}  

\section{The multi-class case}
        \label{sec:multiclass}

We now consider the multi-class case, with $K>2$ classes. Here we are given the predictions of \(m\) classifiers, \(f_i:\mathcal X\to\mathcal Y$, where  $\mathcal{Y}=\{1,\ldots,K\}\). Instead of the class imbalance \(b\), we now have a vector of \(K\) class probabilities $p_k=\Pr(Y=k)$. Similarly, instead of specificity and sensitivity, now each classifier is characterized by a \(K\times K\) confusion matrix \(\psi^i\)
\[ 
\psi_{kk'}^i = \Pr(f_i(X)=k|Y=k') \quad \quad k,k'\in\mathcal Y.
\]

In analogy to  Section \ref{sec:setup}, given only an $m \times n$ matrix of predictions, with elements $f_i(x_j) \in \{1\ldots K\}$,  the problem is to estimate the confusion matrices $\psi^i $ of all classifiers and the class probabilities \(p_k\).

As in the binary case, we make an  assumption regarding the mutual independence of errors made by different classifiers. The precise independence assumption (pairs, triplets or the full set of classifiers) depends on the method employed.

By a simple reduction to the binary case, we now present a partial solution to this problem. We develop a method to consistently estimate the class probabilities \(p_k\) and the diagonals of the confusion matrices, namely the probabilities $\Pr(f_i(X)=k|Y=k)$. However, we prove that even if the class probabilities are a-priori known, estimating all entries of the \(m\) confusion matrices is not possible via this binary reduction.

To this end, we build upon the methods developed in Sections 
 \ref{sec:psi_eta} and \ref{sec:estimate_b} for binary problems. Consider a split of the group $\mathcal Y=\{1\ldots K\}$ into two non-empty disjoint subsets, \(\mathcal Y=\mathcal A\cup(\mathcal Y\setminus \mathcal A)\),
where $\mathcal{A}\subset \mathcal Y$ is a non trivial subset of 
$\mathcal Y$, with $0<|\mathcal A|<K$. Next, define the binary classifiers $\{ f_i^\mathcal A\}_{i=1}^m$:
\[ f_i^\mathcal A(X) = \left\{ 
\begin{array}{rr}
1 \quad f_i(X) \in \mathcal A \\
-1  \quad f_i(X) \not\in \mathcal A 
\end{array} \right.
\]
Using one of the algorithms described in Section \ref{sec:estimate_b}, 
we estimate the probability  of the group $\mathcal{A}$ 
\[
p^{\mathcal{A}} = \Pr(Y \in \mathcal{A})=\sum_{k \in \mathcal{A}} p_k 
\]
and the sensitivity of each classifier $f_i^\mathcal{A}$ 
by Eq. \eqref{eq:psi_eta_est}. 

In particular, when $\mathcal A =\{k\}$, $p^\mathcal A=p_k$ and $\psi_i^\mathcal A=\psi^i_{kk}$. Hence, by considering all 1-vs.-all splits, 
we consistently and computationally efficiently estimate all class probabilities \(p_k\), and all diagonal entries \(\psi^i_{kk}\).

The following theorem, proven in the appendix, states a negative result, that estimating the full confusion matrix is not possible by this binary reduction method.
\begin{theorem}
Let  $\mu_\mathcal A^i=\mathbb{E}[f_i^\mathcal A]$ and let $R_\mathcal A$ be the covariance matrix of the classifiers $\{f_i^\mathcal A\}_{i=1}^m$.
The inverse problem of estimating the $m$ confusions matrices $\psi^i$,
from the values of  $\{\mu_\mathcal A^i\}_{i=1}^m$ and $R_\mathcal A$ for all possible subsets \(\mathcal A\) of  $\mathcal Y=\{1\ldots K\}$, is in general ill posed with multiple solutions.
\label{Thm:multi}
\end{theorem}

Theorem \ref{Thm:multi} implies that in order to completely estimate the confusion matrices in a multiclass problem, it is necessary to use higher-order dependencies such as tensors or even the full likelihood. Indeed, both  \cite{Zhang_2014} and \cite{Jain_2014} derived such methods based on three-dimensional tensors.

While beyond the scope of this paper, we remark that combining our simpler method with these  tensor-based approaches might produce more accurate algorithms for the multiclass case. 


\section{Experiments}
        \label{sec:simlations}

\subsection{Artificial Data }

First, we demonstrate the performance  of the two  class imbalance estimators on artificial binary data. In the following we constructed an ensemble of $m=10$ classifiers that make independent errors and thus satisfy Eq. (\ref{eq:all_f_independent}). Their sensitivities and specificities were chosen uniformly  at random from the interval $[0.5,0.8]$. Thus, assumption (iii) on the balanced accuracies $\boldsymbol{\pi}$ holds.
The vector of true labels $\mathbf y\in\{\pm 1 \}^n$ was randomly generated according to the class imbalance $b$, and the data matrix $Z$ was randomly generated according to $\mathbf y,\boldsymbol{\psi}$, and $\boldsymbol{\eta}$.
 
Fig. \ref{Fig:estimate_b} presents the accuracy (mean and standard deviation) of the estimates $\hat b$ of the class imbalance, achieved by the two different algorithms of Sections \ref{sec:algoT} and \ref{sec:algo_likelihood},
vs. the number of unlabeled instances \(n\), for several values of the class imbalance, $b=0,0.3,0.6$. As expected, the accuracy of both methods improves with the number of instances. Fig. \ref{Fig:MSE_b}
shows the  mean squared error (MSE) $\mathbb{E}[(\hat b-b)^2]$ vs. the number of samples $n$, on a log-log scale. The linear line with slope $\approx-1$ shows that empirically $\hat b_n=b+\mathcal O_P(1/\sqrt{n})$, in accordance to Lemma \ref{lemma:b_hat_tensor}.
In addition, on simulated data, the restricted likelihood estimator is more accurate than the tensor-based estimator.    

\begin{figure}[!t]
  \centering
  \begin{subfigure}{.4\textwidth}
  \centering
    \includegraphics[width=\textwidth]{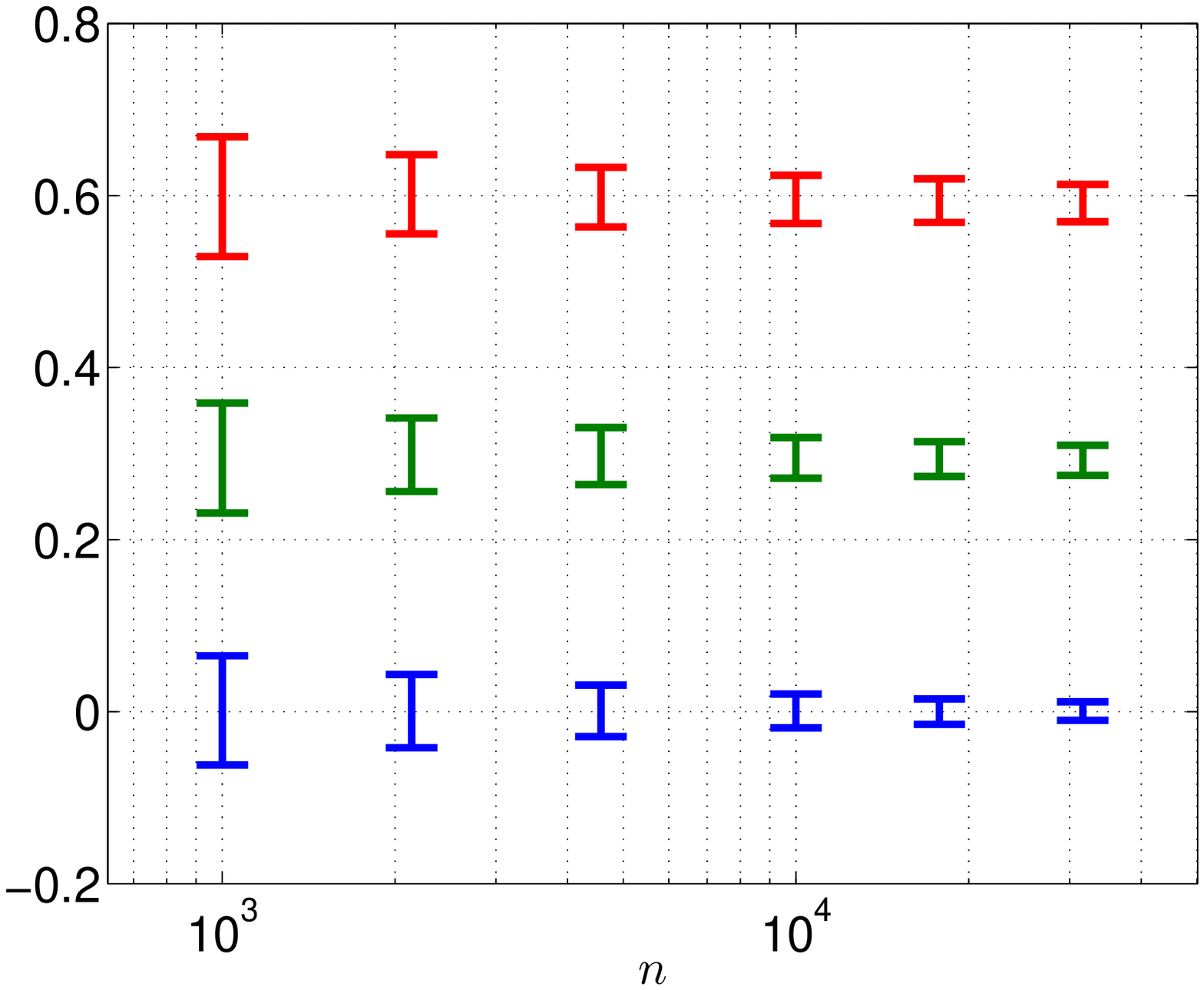}
  \caption{Estimating $b$ via the 3-D tensor \(T\).}
\label{Fig:N_A}
\end{subfigure}
\begin{subfigure}{.4\textwidth}
  \centering
    \includegraphics[width=\textwidth]{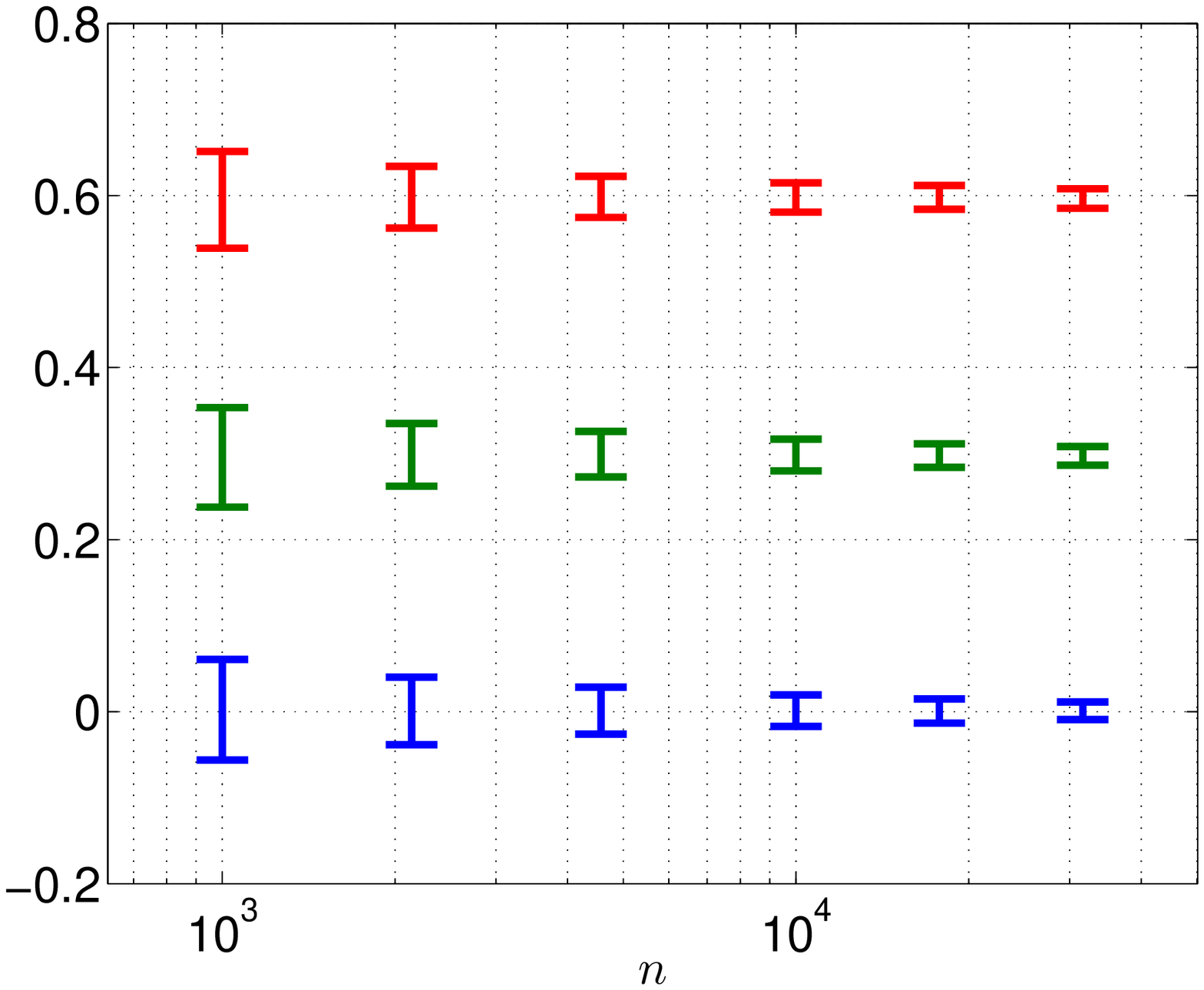}  
  \caption{Estimating $b$ via the restricted likelihood  $\hat G_n$}
\label{Fig:N_B}
\end{subfigure}
\caption{Mean and variance of the tensor-based and likelihood-based class imbalance estimators vs. number of instances $n$, for several values of $b$.}
\label{Fig:estimate_b}
\end{figure}
\begin{figure}
  \centering
  \includegraphics[width=0.35\textwidth]{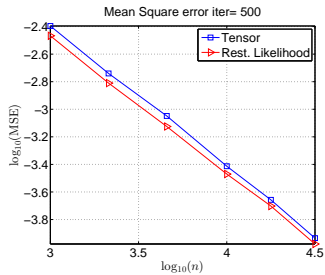}
  \caption{The MSE of the two class imbalance estimators vs. number of samples on a log-log scale.}
\label{Fig:MSE_b}
\end{figure}

\subsection{Real data}
We applied  our algorithms on various binary and multi-class problems using a total of 5 datasets: 4 datasets from the UCI repository \citep{UCI} and the MNIST\ data. Our ensemble consisted of \(m=10\) classification methods implemented in the software package Weka \citep{Weka}. Due to page limits, we present here results only on the 'magic' dataset. Further details on the different datasets, classifiers and additional results appear in the appendix.    

The magic data contains $19,000$ instances with 11 attributes. The task is to distinguish each instance as either background or high energy gamma rays.
Each of the \(m=10\) classifiers was trained on its own  randomly chosen set of  200 instances.  The classifiers were then applied to the whole dataset, thus providing  the $m \times n $ prediction matrix. We compared the results of 4 different unsupervised ensemble methods:
(i) Majority voting; 
(ii) SML of  \cite{Parisi_2014};  
(iii) i-SML as described in section \ref{sec:estimate_b}; and   
(iv) Oracle ML: the MLE formula (\ref{eq:y_ML})\ with the values of $\boldsymbol \psi$ and   $\boldsymbol \eta$, estimated from the full dataset with its labels.

To assess the stability of the different methods, for each dataset we repeated the above simulation $30$ times, each realization with different randomly chosen training sets. Fig. \ref{Fig:compare_real_bar} shows the mean and standard deviation of the balanced accuracy $\pi$ achieved by the four methods on the  'magic' dataset. It shows that on average, i-SML improves upon the SML by approximately $2\%$, and both are significantly better than majority voting.
 Fig. \ref{Fig:compare_real_graph} displays the error rates 
$1-\pi_{\mbox{\tiny\ i-SML}}$ vs. $1-\pi_{\mbox{\tiny SML}}$ for all $30$ realizations.  As all points are below the diagonal,  the improvement over SML\ was consistent in all 30 simulation runs.
As shown in the appendix, similar results, and in particular the improvement of i-SML over SML, were observed also in  all 4 other datasets. 

\section{Summary and Discussion}

In this paper we presented a simple spectral-based approach to estimate, in an unsupervised manner, the accuracies of multiple classifiers, mainly in the binary case. This, in turn, resulted in a novel unsupervised spectral ensemble learner, denoted i-SML. The empirical results on several real data sets attest to its competitive performance in practical situations where clearly the underlying idealized assumptions that all classifiers make independent errors do not hold exactly. 

There are several interesting directions to  extend this work. One possible direction is to relax the strict assumptions of independence of classifier errors across all instances, for example by introducing the concept of instance difficulty. A\ second interesting direction is the construction of novel semi-supervised ensemble learners, when one is given not only the predictions of \(m\) classifiers on a large unlabeled set of instances, but also their predictions on a small set of labeled ones.

\begin{figure}[!t]
  \centering
  \begin{subfigure}{.4\textwidth}
  \centering
  \includegraphics[width=\textwidth]{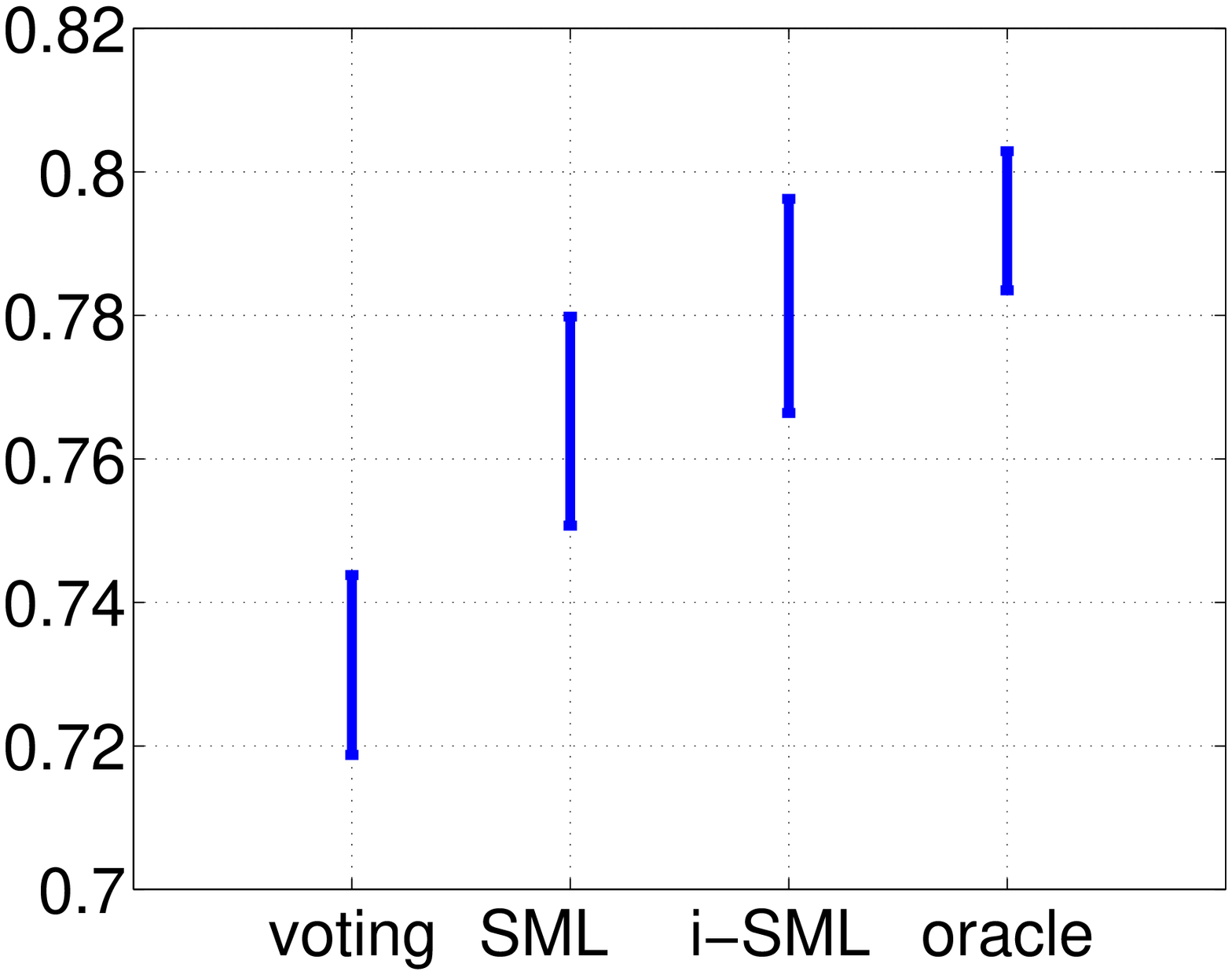}
  \caption{The  balanced accuracies of 4 unsupervised ensemble methods on the magic dataset.}
\label{Fig:compare_real_bar}
\end{subfigure}
\begin{subfigure}{.4\textwidth}
  \centering
  \includegraphics[width=\textwidth]{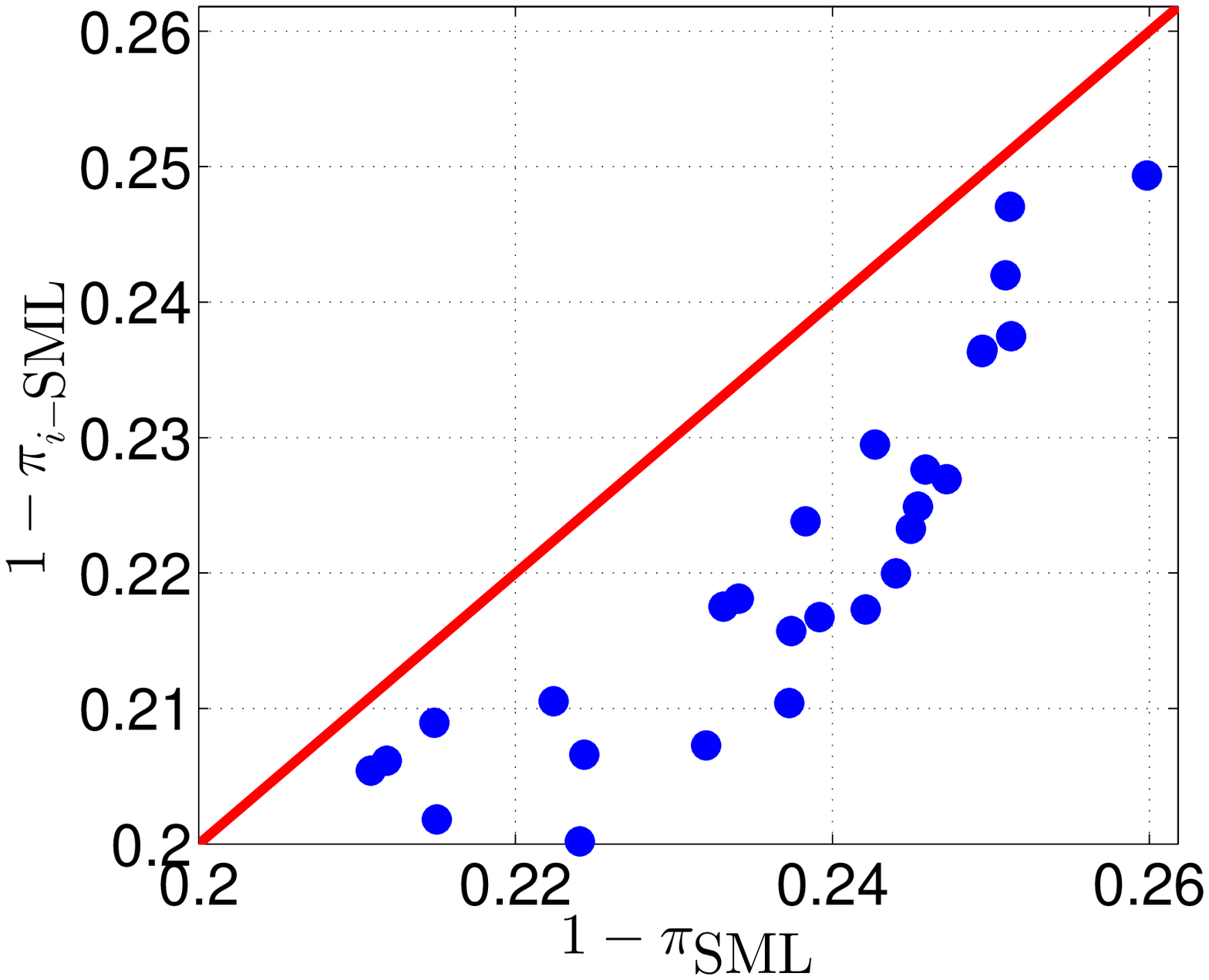}
  \caption{The empirical test error $(1-\pi_{\mbox{i-\tiny{SML}}})$ 
  vs. $(1-\pi_{\mbox{\tiny{SML}}})$ for 30 random realizations.}
\label{Fig:compare_real_graph}
\end{subfigure}
\caption{Comparing 4 unsupervised ensemble learning algorithms, based on  $m=10$ classifiers.}
\label{Fig:compare_real}
\end{figure}

\newpage
\bibliographystyle{plainnat}
\bibliography{bib_unsupervised_classification}

\newpage
\appendix
\section{Estimation of $\psi$ and $\eta$}

\begin{proof}[Proof of Lemma \ref{lem:psi_eta}] We first recall the following formula, derived in \cite{Parisi_2014}, for the vector \(\boldsymbol \mu\) containing the mean values of the \(m\)
classifiers,
\begin{equation}
\boldsymbol{\mu} = 2\boldsymbol{\delta}+b(2\boldsymbol{\pi} -1)
\label{Aeq:mu}
\end{equation}
where $\boldsymbol{\delta}=(\delta_1,\ldots,\delta_m)$ denotes the vector containing half the difference between $\boldsymbol{\psi}$ and $\boldsymbol{\eta}$,
\begin{equation}
\boldsymbol{\delta}=\frac{\boldsymbol \psi-\boldsymbol \eta}{2}.
        \label{eq:delta_def}
\end{equation}
Next, recall from Lemma \ref{lem:v} (also proven in \cite{Parisi_2014}) that the off-diagonal elements of the covariance matrix $R$ correspond to a rank-1 matrix $\mathbf v \mathbf v^T$ where,
\begin{equation}
\mathbf v = \sqrt{1-b^2}(2\boldsymbol{\pi}-1).
\label{Aeq:v}
\end{equation}
Inverting the relation between \(\mathbf v\) and $\boldsymbol \pi$ in Eq. \eqref{Aeq:v} gives
\begin{equation}
\boldsymbol{\pi}  = \frac{1}{2}\left( \frac{\mathbf v}{\sqrt{1-b^2}}+1\right).
\label{Aeq:Pi}
\end{equation}
Plugging \eqref{Aeq:Pi} into \eqref{Aeq:mu}, we obtain  the following expression for the vector $\boldsymbol \delta$, in terms of \(\mathbf v\) and $\boldsymbol \mu$,
\begin{equation}
\boldsymbol \delta= \frac{1}{2}\left(\boldsymbol \mu-b \frac{\mathbf v}{\sqrt{1-b^2}}\right).
\label{Aeq:delta}
\end{equation}
Combining \eqref{eq:delta_def}, \eqref{Aeq:Pi} and \eqref{Aeq:delta} we obtain $\boldsymbol \psi(b)$ and $\boldsymbol \eta(b)$,
\begin{eqnarray}
\boldsymbol \psi = \boldsymbol \pi + \boldsymbol \delta = \frac{1}{2}\left(1 + \boldsymbol \mu + \mathbf v \sqrt{\frac{1-b}{1+b}}\right),\nonumber\\
\boldsymbol \eta = \boldsymbol \pi - \boldsymbol \delta = \frac{1}{2}\left(1 - \boldsymbol \mu + \mathbf v \sqrt{\frac{1+b}{1-b}}\right).\nonumber
\end{eqnarray}
\end{proof}

\section{Statistical Properties of $\psi$ and $\eta$}

\begin{proof}[Proof of Lemma \ref{lem:psi_eta_est}]
Eq. \eqref{eq:psi_eta_est} provides an explicit expression for $\hat{\boldsymbol{\psi}}$ and $\hat{\boldsymbol{\eta}}$ as a function of the estimates $\hat{\mathbf{v}}$ and $\hat{\boldsymbol{\mu}}$. The empirical mean $\hat{\boldsymbol \mu}$ is clearly not only unbiased, but by the law of large numbers also a consistent estimate of $\boldsymbol \mu$, and its error indeed satisfies
\[
\hat{\boldsymbol{\mu}} = \boldsymbol{\mu} + \mathcal{O}_P\left(\frac{1}{\sqrt{n}}\right).
\]
The estimate $\hat{\mathbf v}$, computed by one of the methods described in \cite{Parisi_2014} may be biased, but as proven there is still consistent, and assuming at least three classifiers are different than random (in particular, implying that the eigenvalue of the rank one matrix is non-zero), its error also decreases as $\mathcal{O}_P\left(\frac{1}{\sqrt{n}}\right)$,
\[
  \hat{\mathbf{v}} = \mathbf{v} + \mathcal{O}_P\left(\frac{1}{\sqrt{n}}\right).
\]
Given the exact value of the class imbalance \(b\), since the dependency of $\hat{\boldsymbol{\psi}}$ and $\hat{\boldsymbol{\eta}}$ on
$\hat{\mathbf{v}}$ and $\hat{\boldsymbol{\mu}}$ is linear,
it follows that both are also consistent and that their estimation error is  $\mathcal{O}_P\left(\frac{1}{\sqrt{n}}\right)$.
\end{proof}

\section{ The joint covariance tensor $T$}
\begin{proof}[Proof of Lemma \ref{lem:T}]
To simplify the proof, we first introduce the following linear transformation to the original classifiers,
\[
\tilde f_i(x)=\frac{f_i(x)+1}{2}.
\]
Note, that the output space $\mathcal{Y}$ of the new classifiers is $\{0,1\}$, with class probabilities equal to  $1-p$ and $p$ respectively.
Let us also denote by $\tilde \eta_i$ and $\tilde \psi_{i}$ the following probabilities,
\[
\tilde \eta_i = \Pr(\tilde f_i(x)=1|Y=0), 
\tilde \psi_i = \Pr(\tilde f_i(x)=1|Y=1) .
\]
Note that $\tilde \eta_i$ is not the specificity of classifier $i$, but rather its complement, $\tilde \eta_i = 1-\eta_i$.

The mean of  classifier $\tilde f_i$, denoted  $\tilde \mu_i$, is given by
\begin{equation}
        \label{eq:tilde_mu}
\tilde \mu_i = \mathbb{E}[\tilde f_i(X))] = \Pr(\tilde f_i(X)=1) =  p\tilde \psi_i +(1-p)\tilde \eta_i
\end{equation}
Next, let us calculate the (un-centered)\ covariance between two different classifiers $i  \neq j$,
\begin{multline}
\mathbb{E}[\tilde f_i(X)\tilde f_j(X)] = \Pr(\tilde f_i(X) = 1,\tilde f_j(X) = 1)  \\ 
= p\tilde \psi_i\tilde \psi_j+(1-p)\tilde \eta_i \tilde \eta_j 
\label{Aeq:mu_calc}
\end{multline}
Last, the joint covariance between 3 different classifiers $i \neq j \neq k$ is given by
\begin{eqnarray}
\mathbb{E}[\tilde f_i(X)\tilde f_j(X) \tilde f_k(X)]\!\!\! &=&\! \!\! 
 \Pr(\tilde f_i(X) \!=\!\tilde f_j(X)\!=\!\tilde f_k(X)\!=\!1)\nonumber \\ & =&\!\!\! p\tilde \psi_i\tilde \psi_j\tilde \psi_k+(1-p)\tilde \eta_i \tilde \eta_j \tilde \eta_k 
\label{Aeq:3dCorr}
\end{eqnarray}

The first step in calculating the joint covariance tensor of the original classifiers is to note that \(f_i=2\tilde f_i-1\) and $\mu_i = 2\tilde \mu_i-1$. Hence,
\[
T_{ijk} = \mathbb{E}[(f_i(X)-\mu_i)(f_j(X)-\mu_j)(f_k(X)-\mu_k)] = 8 \tilde T_{ijk}
\]
where
\[
\tilde T_{ijk} = \mathbb{E}[(\tilde f_i(X)-\tilde\mu_i)(\tilde f_j(X)-\tilde\mu_j)(\tilde f_k(X)-\tilde\mu_k)].
\]
Upon opening the brackets, the latter can be equivalently written as
\begin{multline}
\tilde T_{ijk} =\mathbb{E} \left[ \tilde f_i(X)\tilde f_j(X)\tilde f_k(X)\right]\\
        -  \tilde \mu_i \mathbb{E}  \left[ \tilde f_j(X) \tilde f_k(X)\right]
         -  \tilde \mu_j \mathbb{E}  \left[ \tilde f_i(X) \tilde f_k(X)\right]
                        \\
     - \tilde\mu_k \mathbb{E}  \left[ \tilde f_i(X) \tilde f_j(X)\right] +2\tilde \mu_i \tilde \mu_j \tilde \mu_k
\label{eq:Tensor_def}
\end{multline}
Plugging  \eqref{eq:tilde_mu},\eqref{Aeq:mu_calc} and \eqref{Aeq:3dCorr} into \eqref{eq:Tensor_def} we get,
\begin{multline}
\tilde T_{ijk}= p\tilde \psi_i\tilde \psi_j\tilde \psi_k
        + (1-p)\tilde \eta_i \tilde\eta_k \tilde \eta_j -\\
        \left(p\tilde \psi_i +(1-p) \tilde \eta_i \right)
        \left(p\tilde\psi_j\tilde\psi_k+(1-p)\tilde \eta_j \tilde \eta_k\right)-
         \\
        \left(p\tilde \psi_j +(1-p) \tilde \eta_j \right)\!
        \left( p\tilde \psi_k\tilde \psi_i +(1-p)\tilde \eta_k \tilde \eta_i\right)-\\
\left(p\tilde \psi_k +(1-p) \tilde \eta_k \right)\!
\left( p\tilde \psi_i\tilde \psi_j +(1-p)\tilde \eta_i \tilde \eta_j \right)+
         \\
         2 \left(p\tilde \psi_i +(1-p) \tilde \eta_i \right)\left( p\tilde \psi_j +(1-p) \tilde \eta_j\right)
\left( p\tilde \psi_k +(1-p) \tilde \eta_k\right) \notag
\end{multline}
Opening the brackets and collecting similar terms yields
\begin{multline}
\tilde T_{ijk}  =  (p-3p^2+2p^3)\tilde \psi_i\tilde \psi_j\tilde \psi_k + \\
 \left(2p^2(1-p)-p(1-p) \right) \left( \tilde \eta_i\tilde \psi_j\tilde \psi_k +\tilde \eta_j\tilde \psi_k\tilde \psi_i + \tilde \eta_k\tilde \psi_i\tilde \psi_j \right) +
\\
\left(2p(1-p)^2-p(1-p) \right) \left( \tilde \eta_i\tilde \eta_j\tilde \psi_k +\tilde \eta_j\tilde \eta_k\tilde \psi_i + \tilde \eta_k\tilde \eta_i\tilde \psi_j \right)+ \\
\left( (1-p) -3(1-p)^2+2(1-p)^3 \right) \tilde \eta_i\tilde \eta_k\tilde \eta_j. \notag
\end{multline}
Note that all polynomials in \(p\) in the above expression are equal to \(\pm p(1-p)(1-2p)\). Hence,
\begin{multline}
\tilde T_{ijk} \!=\!
p(1-p)(1-2p)(\tilde \psi_i\tilde \psi_j\tilde \psi_k - \tilde \eta_i\tilde \psi_j\tilde \psi_k - \tilde \eta_j\tilde \psi_k\tilde \psi_i - \\
\tilde \eta_k\tilde \psi_i\tilde \psi_j+\tilde \eta_i\tilde \eta_j\tilde \psi_k +\tilde \eta_j\tilde \eta_k\tilde \psi_i + \tilde \eta_k\tilde \eta_i\tilde \psi_j -\tilde \eta_i\tilde \eta_k\tilde \eta_j )        
\end{multline}
Finally, replacing $\tilde \psi_i=\psi_i$, $\tilde \eta_i=1-\eta_i$ and
$p=\frac{1+b}{2}$, yields
 \begin{multline}
 T_{ijk} = -2b(1-b^2)(\psi_i +\eta_i-1)(\psi_j +\eta_j-1)(\psi_k +\eta_k-1)
  \\
  = -2b(1-b^2)(2\pi_i-1)(2\pi_j-1)(2\pi_k-1). \notag
 \end{multline}
 \end{proof}

\begin{proof}[Proof of Lemma \ref{lemma:b_hat_tensor}] To prove that $\hat b_n$ is consistent with an asymptotic error $\mathcal O_P(1/\sqrt{n})$, we first recall that according to \cite{Parisi_2014}, it follows that
\[
  \hat{\mathbf{v}} = \mathbf{v} + \mathcal{O}_P\left(\frac{1}{\sqrt{n}}\right).
\]
By its definition, each entry of $\hat T_{ijk}$ also incurs an error of $O_P(1/\sqrt{n})$.
Hence, by the delta method, the estimate $\hat\alpha$ of Eq. (\ref{eq:hat_alpha}), being a least squares minimizer, also satisfies
\[
\hat \alpha = \alpha + \mathcal O_P(1/\sqrt{n}). 
\] 
Since $\hat b_n$ is found by the smooth relation of Eq. (\ref{eq:b}), again by the delta method, $\hat b_n = b+\mathcal O_P(1/\sqrt{n})$. Finally, the fact that the corresponding
estimates $\hat \psi_i$ and $\hat\eta_i$ also have errors $\mathcal O_P(1/\sqrt{n})$ follows by standard application of the delta method to 
Eq. (\ref{eq:psi_eta_est}), where all quantities $\hat{\boldsymbol \mu},
\hat{\mathbf v}$ and $\hat b$ have errors $\mathcal O_P(1/\sqrt{n})$. 
\end{proof}

\paragraph{Dependence of estimated parameters on number of classifiers and their accuracies.} Beyond the fact that $\hat\alpha$ and consequently $\hat b_n,\hat{\boldsymbol \psi},\hat{\boldsymbol \eta}\) are all $\mathcal O(1/\sqrt{n})$ consistent, it is of interest to study the dependence of these estimates on the number of classifiers and their accuracies. To this end, we first prove the following simple result.

\begin{lemma} Let $\hat\alpha$ be the estimate of $\alpha$ in Eq. (\ref{eq:hat_alpha}). Then asymptotically as $n\to\infty$, its estimation error is given by 
\begin{equation}
\hat \alpha- \alpha = 
\frac{\langle \hat T-T,\mathbf v^{\otimes 3}\rangle }{\langle \mathbf v^{\otimes 3},{\bf v}^{\otimes 3}\rangle} -
\alpha\frac{\langle \hat{\bf v}^{\otimes 3}-\mathbf v^{\otimes 3},{\bf v}^{\otimes 3}\rangle}{\langle \mathbf v^{\otimes 3},{\bf v}^{\otimes 3}\rangle}
+O_P\left(\frac1{{n}}\right)
        \label{eq:err_alpha}
\end{equation}
where ${\mathbf v}^{\otimes 3}=\mathbf v\otimes\mathbf v\otimes \mathbf v$, and for any two tensors $T,S$, $\langle T,S\rangle = \sum_{i<j<k} T_{ijk}S_{ijk}$. 
\end{lemma} 
\begin{proof} 
The minimizer of Eq. (\ref{eq:hat_alpha}) is given by
\[
\hat\alpha=\frac{\langle \hat T,\hat{\bf v}^{\otimes 3}\rangle}
{\langle \hat{\bf v}^{\otimes 3},\hat{\bf v}^{\otimes 3}\rangle}
\]
According to \cite{Parisi_2014}, as $n\to\infty$, 
the estimate $\hat{\bf v}$ is $O(1/\sqrt{n})$ consistent, 
namely $\hat{\bf v}={\bf v}+\delta{\bf v}$, where $\delta{\bf v}=O_P(1/\sqrt{n})$.
Writing $\hat T = T + (\hat T-T)$ where the latter is also $O_P(1/\sqrt{n})$
and inserting these into the expression for \(\hat\alpha\)\ above gives that
\[
\hat\alpha = \frac{\langle T,{\bf v}^{\otimes 3}\rangle + \langle \hat T-T,{\bf v}^{\otimes 3}\rangle + \langle T,\hat{\bf v}^{\otimes 3}-{\bf v}^{\otimes 3}\rangle+O_P(1/n)}{\langle{\bf v}^{\otimes 3},{\bf v}^{\otimes 3}\rangle
+2\langle{\bf v}^{\otimes 3},{\hat{\bf v}}^{\otimes 3}-{\bf v}^{\otimes 3}\rangle+O_P(1/n)}
\]
Next, recall that \(T=\alpha{\bf v}^{\otimes 3}\). Now, keeping only the leading order error terms yields Eq. (\ref{eq:err_alpha}).
\end{proof}
According to Eq. (\ref{eq:err_alpha}), the estimation error depends on the statistical properties of the deviations $\hat{\bf v}-{\bf v}$ and  $\hat T-T$ and their correlations. While these are quite complicated, we may gain insight by looking at some particular instances. Assume for simplicity that all classifiers have comparable accuracies. Then,  $\langle {\bf v}^{\otimes 3},{\bf v}^{\otimes 3}\rangle \propto m(m-1)(m-2)/6 \cdot (2\pi-1)^6$. Hence, the estimation error in $\hat \alpha$ should decrease with the number of classifiers. Moreover, for a balanced problem with \(b=0\) and hence \(\alpha=0\), to leading order, the errors in $\hat\alpha$ and consequently also in $\hat b_{n}$ should not depend on the errors in estimating the eigenvector $\bf v$.
Figure \ref{fig:MAE_b_vs_n} shows this empirically. The \(x\)-axis is the number of classifiers,
the $y$-axis is the mean absolute deviation $\mathbb{E}[|\hat b_n-b|]$ (MAE), both on a log scale.
We considered two values $b=0$ and $b=0.3$, and for each value of $b$ we plotted two curves, one corresponding to the estimate $\hat b$ computed from $\hat\alpha$ based on $\hat{\bf v}$, and the second, an ``oracle" one, where $\hat\alpha$ is estimated using the true $\bf v$. Indeed, for $b=0$ both curves nearly coincide, in accordance to Eq. (\ref{eq:err_alpha}). In this simulation, all classifiers had a balanced accuracy in the range $[0.69,0.71]$, and $n=10,000$. 
These results suggest that it is potentially profitable to estimate the eigenvector ${\bf v}$ and the scalar $\alpha$ {\em jointly} from both the covariance matrix $\hat R$ and the tensor $\hat T$, and not separately as done in the present paper. This, as well as a more detailed study of the estimation errors are issues beyond the scope of the current work.  
\begin{figure}[t]
\centering
\includegraphics[width=0.45\textwidth]{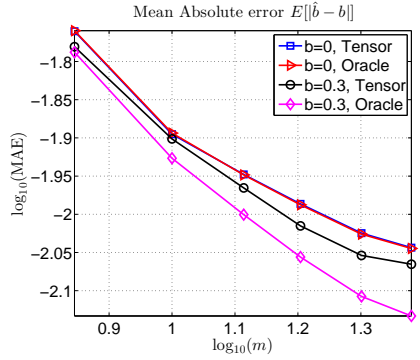}
\caption{Mean absolute error for the tensor based method, $\mathbb{E}[|\hat b_n-b|]$ vs. number of classifiers $m$, on log-log scale. }
        \label{fig:MAE_b_vs_n}
\end{figure}

\section{The Restricted Likelihood Function}

\begin{proof}[Proof of Theorem \ref{Thm:g}] By definition, the function $\hat g_n(\mathbf{f}(x)|\tilde b)$ in Eq. (\ref{eq:g_def}) is the log-likelihood of the observed vector \(\mathbf{f}(x)\) of predicted labels at an instance $x$, assuming the class imbalance is $\tilde b$ and using the estimates $\hat{\boldsymbol{\psi}}$ and $\hat{\boldsymbol{\eta}}$ for the sensitivities and specificities of the $m$ classifiers. 

Under the assumption that all classifiers make independent errors, the expression for $\Pr(\mathbf{f}(x)|\hat{\boldsymbol{\psi}},\hat{\boldsymbol{\eta}},\tilde b)$ is given by
\begin{multline}
        \label{eq:Prob_f_b}
\Pr(\mathbf f|\tilde b) = \Pr (y=1|\tilde b)\Pr(\mathbf f|\tilde b,y=1) +\\ \Pr (y=-1|\tilde b)\Pr(\mathbf f|\tilde b,y=-1)=  \\
\left(\tfrac{1+\tilde b}{2}\right)\prod_{i=1}^m\hat\psi_i^{\frac{1+f_i(x)}{2}}
(1-\hat\psi_i)^{\frac{1-f_i(x)}{2}}+\\
\left(\tfrac{1-\tilde b}{2}\right)
 \prod_{i=1}^m\hat\eta_i^{\frac{1-f_i(x)}{2}}(1-\hat\eta_i)^{\frac{1+f_i(x)}{2}}
  \end{multline}
We first prove Eq. (\ref{eq:b_max_Eg}),\ that upon using the exact log-likelihood function \(g(\mathbf{f}|\tilde b\)), its mean is maximized at the true value  $b$.
 To this end, we write the expectation explicitly,
\begin{eqnarray}
\mathbb{E}[g(\mathbf f|\tilde b)] &=& \sum_{\mathbf f \in \{-1,1\}^{m}}\Pr(\mathbf f|b)g(\mathbf f|\tilde b) \nonumber\\
&=& \sum_{\mathbf f \in \{-1,1\}^{m}}\Pr(\mathbf f|b)\log \Pr(\mathbf f|\tilde b)
\label{eq:b_exp}
\end{eqnarray}
Note the difference between the assumed class imbalance  $\tilde b$, which appears inside the logarithm, and its true value $b$, over which we take the expectation.

To prove Eq. \eqref{eq:b_max_Eg}, let us first present the following auxiliary lemma, which can be easily proved  using Lagrange multipliers.

\begin{lemma}
Consider the following function of \(k\) unknown variables
 $\{c_i\}_{i=1}^k$,
\begin{equation}
h(\{c_i\}_{i=1}^k|\{a_i\}_{i=1}^k)= \sum_{i=1}^k a_i \log(c_i).
\label{lem:a_c}
\end{equation}
where $\{a_i\}_{i=1}^k$ are $k$ non-negative constants. Under the constraints that $\sum_{i=1}^k c_i = 1$, and \(c_{i}\geq 0\), the function $h$ has a global maxima at $c_i = a_i$ for all $i$. 
\end{lemma}
We use this lemma with \(k=2^{m}\) and the following set of $2^m$ constants $a_\mathbf{f}(b)=\Pr(\mathbf f|b)$, over all possible \(m\)-dimensional vectors \(\mathbf{f}\in\{-1,1\}^m\), and the $2^m$ variables $c_\mathbf{f} = \Pr(\mathbf f|\tilde b)$.
The expectation of $g$ is now equal to
\begin{equation}
G(\tilde b) = \mathbb{E}[g(\mathbf f|\tilde b)] = \sum_{i=1}^{2^m}a_i\log(c_i)
\end{equation} 
By Eq. \eqref{lem:a_c}, over all possible choices of \(c_{i}\), the expectation attains its maxima at $c_i = a_i$ for all $i$. Since at $\tilde b=b$, the corresponding probabilities \(\Pr(\mathbf{f}|\tilde b=b)=a_\mathbf{f}\), Eq. (\ref{eq:b_max_Eg}) follows.  

Next, we wish to prove that $\hat b_n\to b$ in probability. To this end, we follow the approach outlined in \cite{Newey_1991}, and prove the following uniform convergence in probability of \(\hat G_n\) to $G$, 
\[
\sup_{\tilde b \in[-1+\delta,1-\delta]} |\hat G_n(\tilde b)-G(\tilde b)| = o_P(1)
\]
This equation, coupled with the equicontinuity of $G$ implies the convergence in probability of the maximizer of $\hat G_n$ (namely $\hat b_n$) to that of $G$, which by Eq. (\ref{eq:b_max_Eg}) is $b$. 

As proved in \cite[Theorem 2.1]{Newey_1991}, this uniform convergence in probability is satisfied if and only if there is pointwise convergence of $\hat G_n(\tilde b)$ to $G(\tilde b)$, and $\hat G_n(\tilde b)$ is stochastic equicontinuous. Fortunately, a sufficient condition for the latter property is that $\hat G_n(\tilde b)$ is continuously differentiable and its derivative bounded, see \cite{Newey_1991} Corollary 2.2
and discussion after it. 

In our case, since $\hat G_n(\tilde b)=1/n \sum_i \hat g_n(\mathbf{f}(x_i)|\tilde b)$, it suffices to prove that for any vector $\mathbf{f}$, the function $\hat g_n(\mathbf{f} |\tilde b)$ is continuously differentiable with a bounded derivative. First note that by their definition, Eq. \eqref{eq:psi_eta_est}, the functions \(\hat\psi_i(\tilde b)\) and $\hat\eta_i(\tilde b)$ are continuously differentiable with bounded derivative for all $\tilde b\in[-1+\delta,1-\delta]$. Next, under the assumptions of the theorem, that \(\psi_i\) and $\eta_i$
are $\epsilon$ bounded from 0 and from 1, and hence also their estimates
can be restricted to  $\epsilon<\hat\psi_i,\hat\eta_i<1-\epsilon$,   the term inside the logarithm in Eq. (\ref{eq:g_def})\ is bounded away from zero. Hence, by its definition $\hat g_n$ satisfies the required condition.
\end{proof}

\section{Ambiguity in the Multi-Class Case}

\begin{proof}[Proof of Theorem \ref{Thm:multi}]
For simplicity, let us assume that all \(K\) class probabilities are equal, $p_i=\frac{1}{K}$ for $i=1,\ldots,K$. 
Let $f_i$ be the set of original classifiers with confusion matrices $\{\psi^{i}\}_{i=1}^m$. We shall now construct another set of classifiers with different confusion matrices that nonetheless lead to the \textit{same} values  $\mu_\mathcal A^i$ and $R_\mathcal A$ for all subsets $\mathcal A$.

To this end, assume that all entries of the first confusion matrix $\psi^1$ are strictly positive and strictly smaller than one. Consider a second set of confusion matrices $\{ \tilde \psi^i\}_{i=1}^m$ identical to the first, except for the following six changes in $\psi^{1}$: For three fixed indices $j\neq k\neq l$, let
\begin{eqnarray}
\tilde \psi^1_{jk} = \psi^1_{jk}+\Delta & \tilde \psi^1_{kj} = \psi^1_{kj}-\Delta \notag\\
\tilde \psi^1_{lj} = \psi^1_{lj}+\Delta & \tilde \psi^1_{jl} = \psi^1_{jl}-\Delta \notag\\
\tilde \psi^1_{kl} = \psi^1_{kl}+\Delta & \tilde \psi^1_{lk} = \psi^1_{lk}-\Delta \notag
\end{eqnarray}
where $\Delta$ is sufficiently small so that all entries of $\tilde \psi^1$ are in $[0,1]$. 

Note that the new matrix \(\tilde\psi^1\) is
a valid confusion matrix, since
for any column \(r\in\{1,\ldots,K\}\)
\[
\sum_{i=1}^K \tilde \psi_{ir}^1= 1.
\]
Let \(\tilde f_1\) be the classifier corresponding to the modified matrix $\tilde \psi^1$. Next, note that the first order statistics of $\tilde f_1$\ and of $f_1$ are unchanged. Indeed, by definition
\[
\Pr(\tilde f_1(X) = r)= \frac{1}{K} \sum_{i=1}^K \tilde \psi_{ri}^1 
\]
If $r\notin\{j,k,l\}$, then $\tilde\psi^1_{ri}=\psi^1_{ri}$ and thus
\begin{equation}
\Pr(\tilde f_1(X) = r) = \Pr(f_1(X) = r)
        \label{eq:ftilde_f}
\end{equation}
If $r\in\{j,k,l\}$, then by construction, in the $r$-th row of $\tilde\psi^1$ there are precisely two modified entries, one increased by $\Delta$ and the other reduced by $\Delta$, so overall the above equation still holds. 
Eq. (\ref{eq:ftilde_f}) directly implies that  
$\tilde \mu_\mathcal A^1=\mu_\mathcal A^1$ for all subsets \(\mathcal A\). 

Next, let us
show that the covariance matrices \(R_{\mathcal A}\) also remain unchanged. 
Recall that the entries of  $R_\mathcal A$ are determined by the values $\psi_\mathcal A^1 \ldots \psi_\mathcal A^m$ and $\eta_\mathcal A^1 \ldots \eta_\mathcal A^m$. 
Hence, it suffices to show that for all subsets $\mathcal A$
\begin{equation}
\tilde\psi^1_\mathcal A=\psi^1_\mathcal A
\quad\mbox{and}
\quad
\tilde\eta^1_\mathcal A=\eta^1_\mathcal A
        \label{eq:psi_tilde_psi}
\end{equation}
To this end, recall that by definition
\[
\tilde \psi_\mathcal A^1 = \frac{1}{K}\sum_{i,i' \in \mathcal{A}}\tilde  \psi_{ii'}^1
\quad\mbox{and}
\quad
\tilde \eta_\mathcal A^1 = \frac{1}{K}\sum_{i,i' \notin \mathcal{A}}\tilde  \psi_{ii'}^1
\]
First consider the case $|\mathcal A\cap\{j,k,l\}|=0$. Here, all relevant entries in the sum for $\tilde\psi^1_\mathcal A$ are unchanged. In contrast, the sum for $\tilde\eta^1_\mathcal A$ includes all six modified entries. Both sums remain unchanged, and so Eq. (\ref{eq:psi_tilde_psi}) holds. 

The proof for the other cases, where $\mathcal A\cap\{j,k,l\}\ne \emptyset$ follows similar arguments. 

To conclude, both $\{\psi^i\}_{i=1}^m$ and $\{\tilde\psi^i\}_{i=1}^m$ have the same values $\mu^{i}_\mathcal A$ and covariance matrices $R_\mathcal A$. 
\end{proof}

\section{Ensemble of Machine Learning Classifiers }

Table \ref{Table_1} presents the 10 different classifiers used in our experiments. For each dataset, each classifier was trained with $200$ different (randomly chosen) instances.
\vspace{1em}
\begin{table}[t]

\begin{tabular}{|l| l|}
\hline
\textbf{classifier} & \textbf{Weka library} \\
\hline
IBk - K nearest &\\ neighbours, $K=1$ & lazy.IBk\\
\hline
KStar - Instance& \\based classifier & lazy.KStar\\
\hline
J48 - Decision tree  & trees.J48\\
\hline
PART - Partial decision& \\trees classifier   & rules.PART \\
\hline
LMT - Logistic model&\\ trees   & trees.LMT\\
\hline
Random forest - &\\with $n=10$ trees   & trees.RandomForest\\
 \hline
Logistic Regression   & functions.SimpleLogistic\\
 \hline
Decision Stump - &\\One level decision tree   & trees.DecisionStump\\
\hline
Sequential Minimal&\\ Optimization   & functions.SMO\\
\hline
NaiveBayes   & bayes.NaiveBayes \\
\hline
\end{tabular}
\caption{10 classification methods implemented in the software package Weka.}
\label{Table_1}
\end{table}

\section{Real Datasets}

We tested our methods on a total of  five datasets, 4  from the UCI repository and the MNIST\ digits data. A short description of each of the datasets is given in Table \ref{Table_UCI}.
A comparison of the performance of various ensemble learners on these
datasets appears in Fig. \ref{Fig:Appendix_Results}.

\begin{table}[t]
\begin{tabular}{ p{10mm} p{30mm} p{15mm} p{15mm}}
\hline
\textbf{dataset} & \textbf{Task} & \textbf{instances} & \textbf{attributes}\\
 \hline
Magic & classifying gamma rays from background noise & $19000 $& 11\\
\hline
Spam & classifying spam from regular mail&$4600$ &$57$\\
\hline
Musk & classifying different types of molecules to be 'musk' or 'non musk'&$6600 $& $88$\\
\hline
Miniboo & distinguish electron neutrinos (signal) from muon neutrinos (background)'&$130000 $& $50$\\
\hline
Mnist & To define a binary problem, we divided the  MNIST data set into two classes as follows: $0-4$ vs. $5-9$&$40000 $&$ 28^2$\\
\end{tabular}
\caption{Properties of datasets from the UCI repository}
        \label{Table_UCI}
\end{table}

\begin{figure*}[!t]  
\label{Fig:N}
\begin{subfigure}{.4\textwidth}
  \centering
  \includegraphics[width=\textwidth]{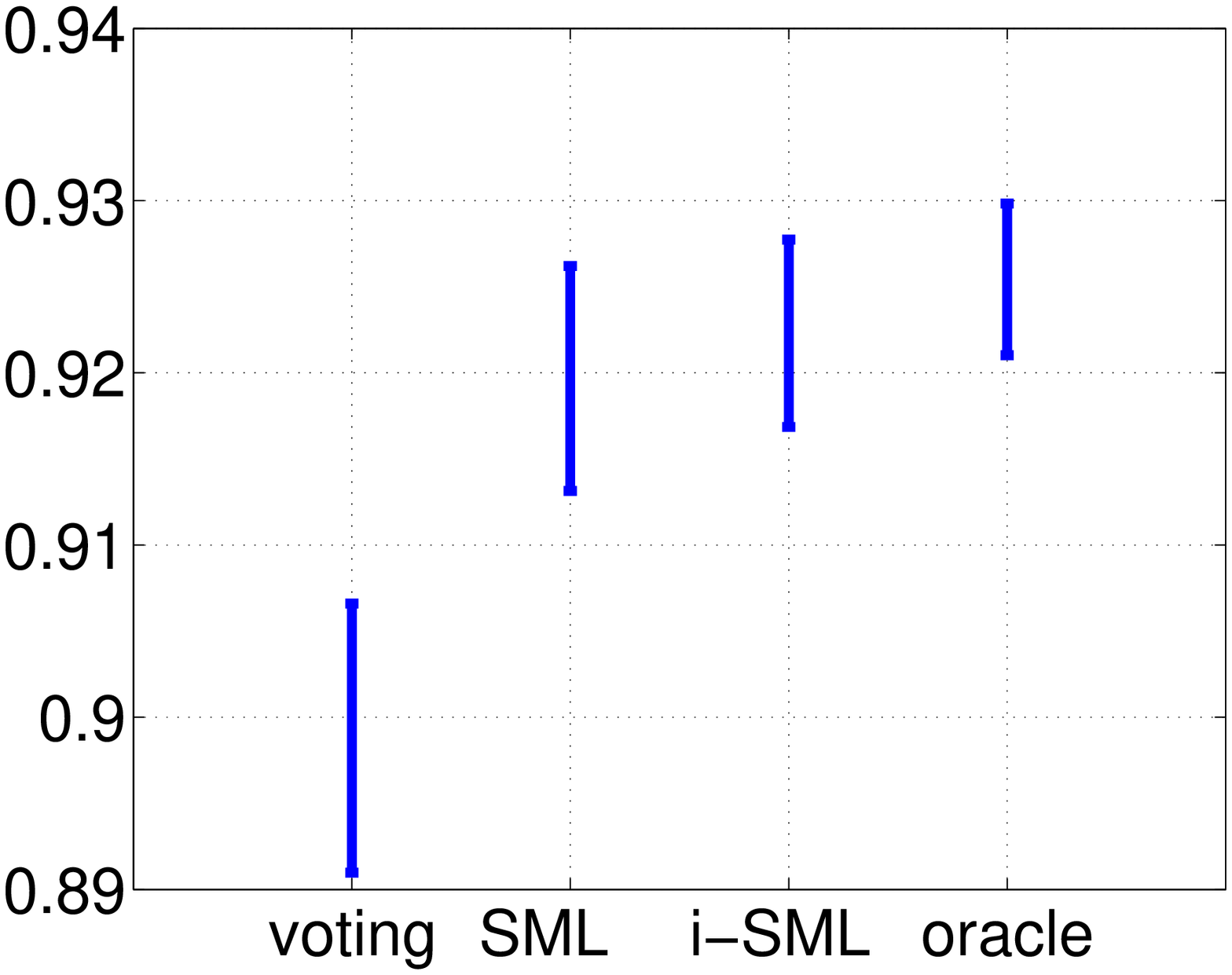}
  \caption{Spam Dataset}
\label{Fig:N_B}
\end{subfigure}
\begin{subfigure}{.4\textwidth}
  \centering
  \includegraphics[width=\textwidth]{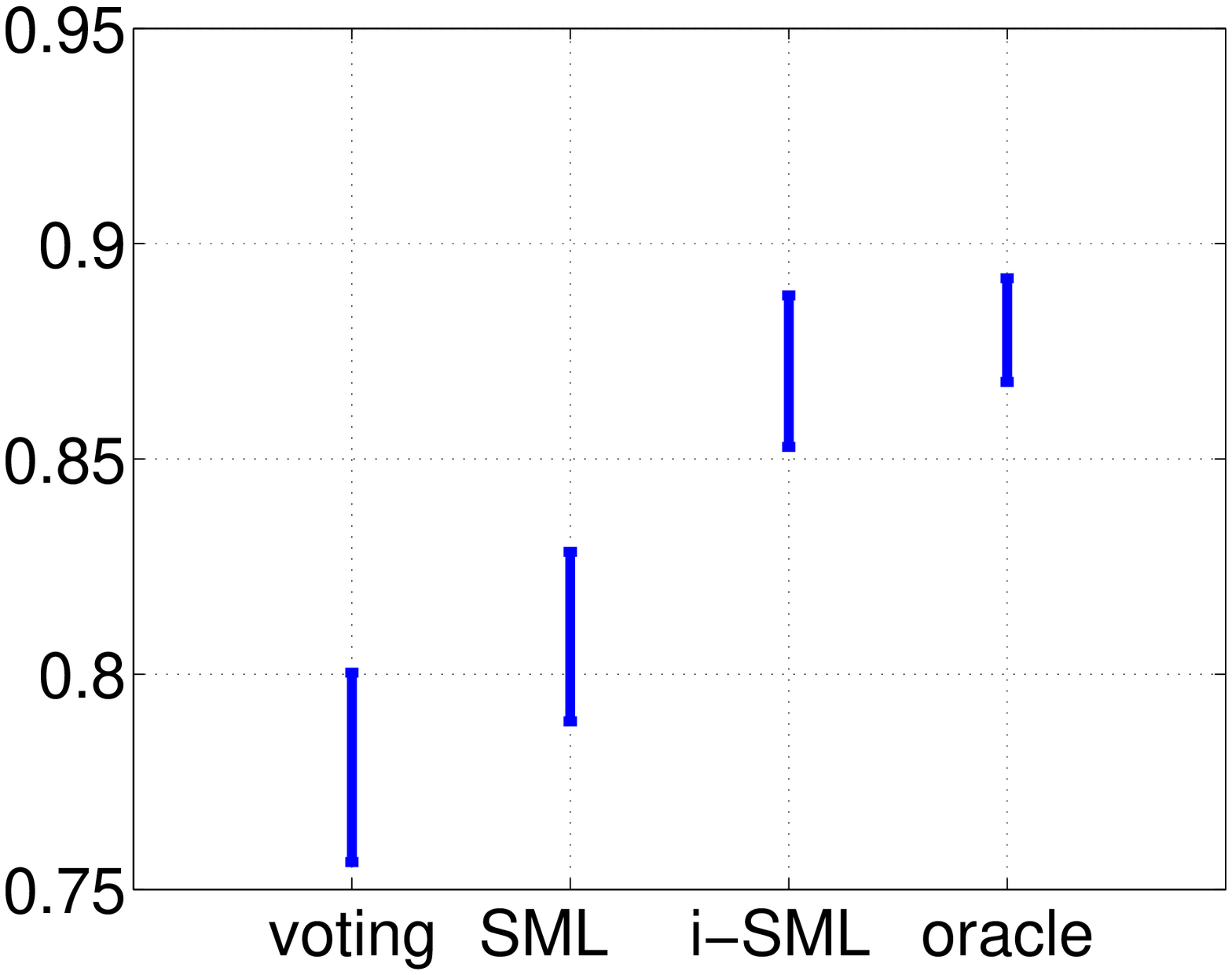}
  \caption{Musk Database}
\label{Fig:N_B}
\end{subfigure}

\begin{subfigure}{.4\textwidth}
  \centering
  \includegraphics[width=\textwidth]{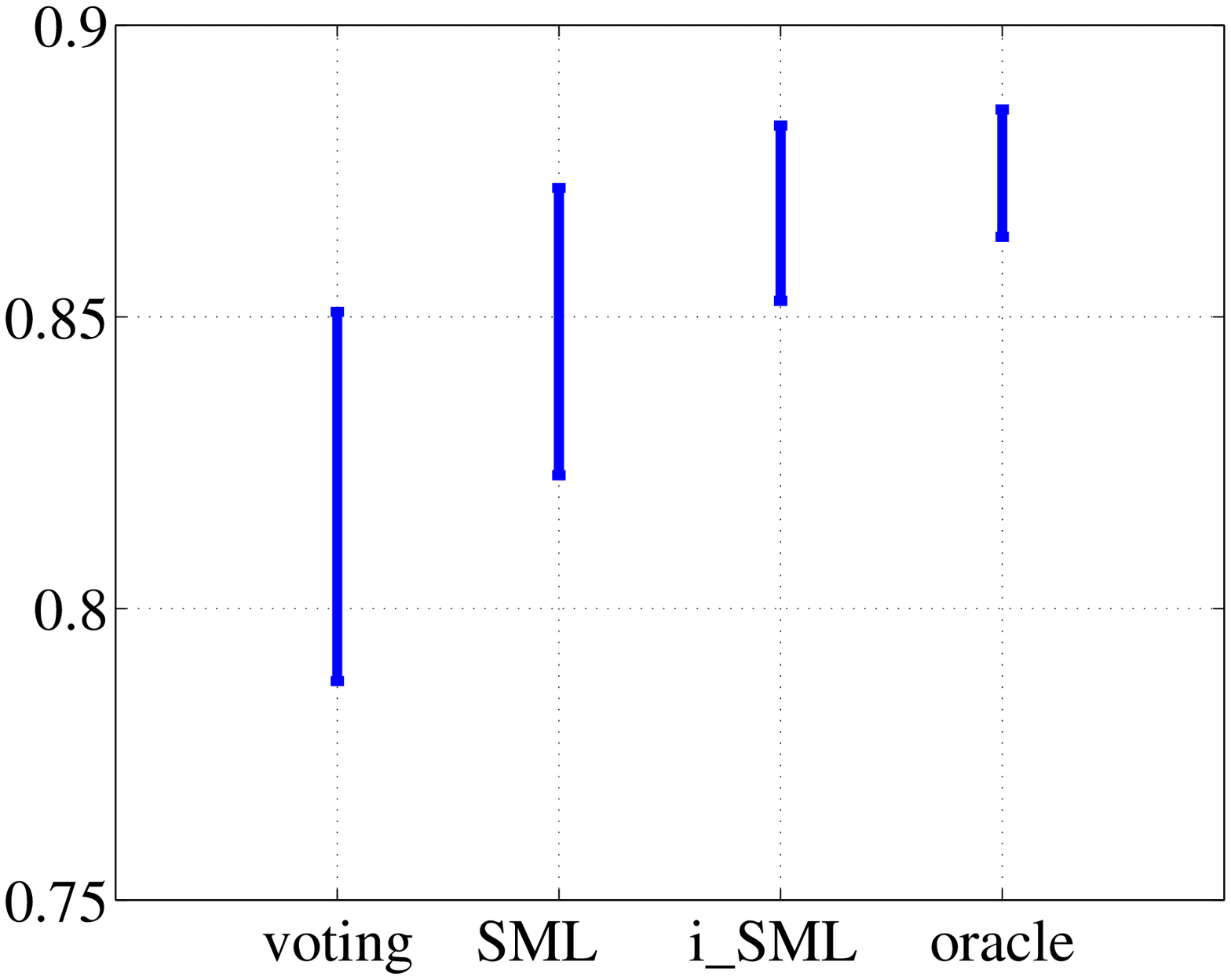}
  \caption{Miniboo Dataset}
\label{Fig:N_B}
\end{subfigure}
  \begin{subfigure}{.4\textwidth}
  \centering
  \includegraphics[width=\textwidth]{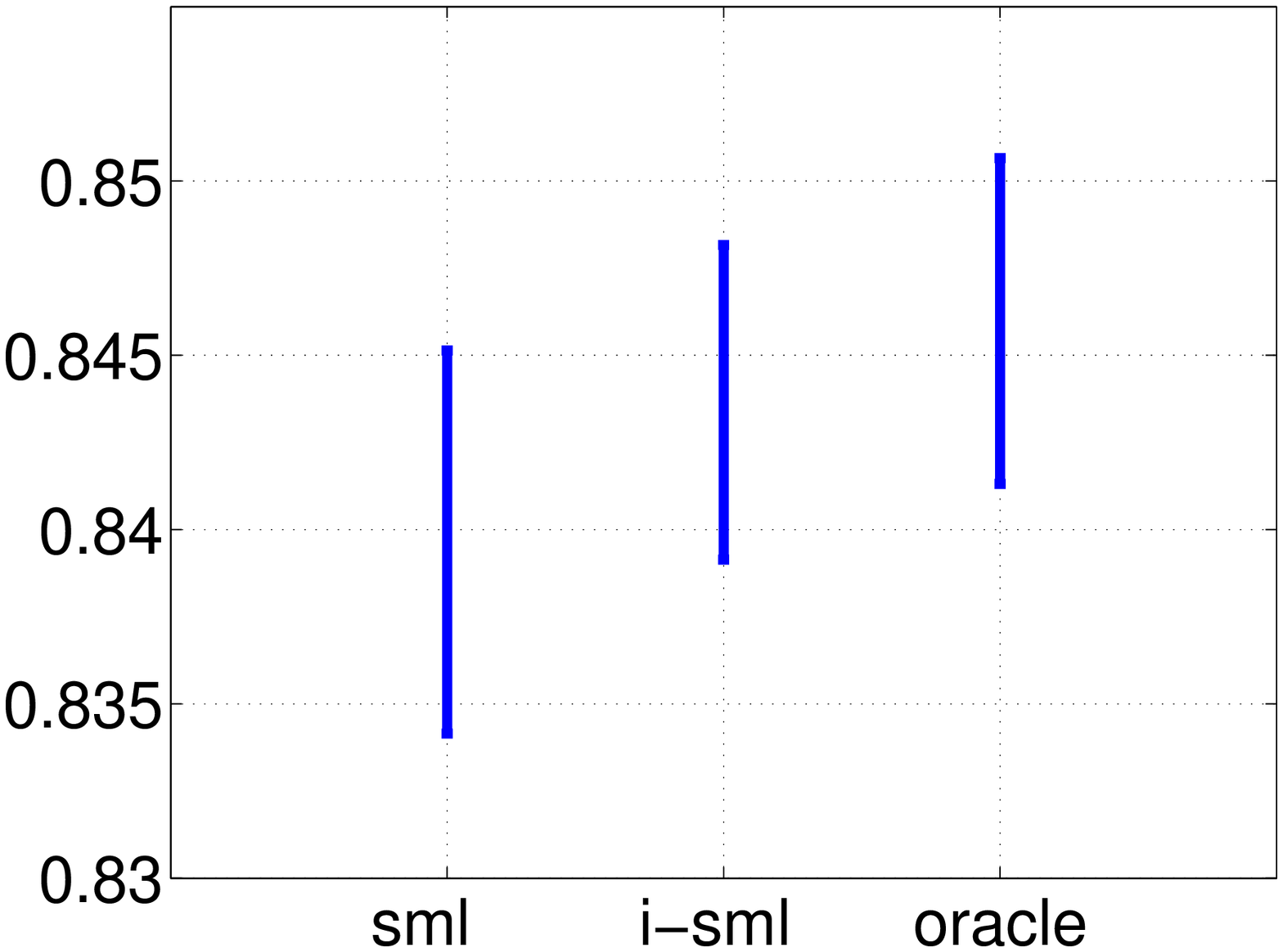}
  \caption{'MNIST' Dataset.}
\label{Fig:N_A}
\end{subfigure}
\caption{The balanced accuracies of 4 unsupervised ensemble learning algorithms, all with  $m=10$ classifiers. In panel \ref{Fig:N_A} we do not show the accuracy of majority voting which was significantly lower than all others. }
        \label{Fig:Appendix_Results}
\end{figure*}

\end{document}